\documentclass{article}

\usepackage{geometry}
\usepackage[utf8]{inputenc} 
\usepackage[T1]{fontenc}    
\usepackage{hyperref}       
\usepackage{url}            
\usepackage{booktabs}       
\usepackage{amsfonts}       
\usepackage{nicefrac}       
\usepackage{microtype}      
\usepackage{comment}

\usepackage{amsfonts}
\usepackage{amsthm,amsmath}
\usepackage{amssymb}
\usepackage{bbm}
\newtheorem{lemma}{Lemma}
\newtheorem{theorem}{Theorem}
\newtheorem{proposition}{Proposition}
\usepackage{dsfont}

\usepackage{hyperref}
\usepackage{enumitem}
\usepackage{mathtools}
\usepackage{array}

\usepackage{microtype}
\usepackage{graphicx}
\usepackage{subfigure}
\usepackage{lscape}
\usepackage{balance}
\usepackage{color}
\usepackage[dvipsnames]{xcolor}
\usepackage{xspace}
\usepackage{multirow}
\usepackage{adjustbox}
\usepackage{algorithm}
\usepackage[square, compress]{natbib}
\usepackage{algpseudocode}

\newcommand{\ra}[1]{\renewcommand{\arraystretch}{#1}}
\newcolumntype{H}{>{\setbox0=\hbox\bgroup}c<{\egroup}@{}}

\newtheorem{definition}{Definition}

\newtheorem{corollary}{Corollary}

\newcommand\calX{\boldsymbol{\mathcal{X}}}
\newcommand\calY{\boldsymbol{\mathcal{Y}}}
\newcommand\calZ{\boldsymbol{\mathcal{Z}}}
\newcommand\calO{\boldsymbol{\mathcal{O}}}

\newcommand\calF{\boldsymbol{\mathcal{F}}}
\newcommand\calD{\boldsymbol{\mathcal{D}}}
\newcommand\calG{\boldsymbol{\mathcal{G}}}

\newcommand\calH{\boldsymbol{\mathcal{H}}}

\newcommand\IR{\mathbb{R}}
\newcommand\by{\boldsymbol{y}}
\newcommand\bz{\boldsymbol{z}}
\newcommand\bx{\boldsymbol{x}}

\newcommand\bbeta{\boldsymbol{\beta}}
\newcommand\balpha{\boldsymbol{\alpha}}

\newcommand\bd{\boldsymbol{d}}

\newcommand\bG{\boldsymbol{G}}

\newcommand\bw{\boldsymbol{w}}
\newcommand\bW{\boldsymbol{W}}

\newcommand\bZ{\boldsymbol{Z}}
\newcommand\bY{\boldsymbol{Y}}

\newcommand\bX{\boldsymbol{X}}
\newcommand\bD{\boldsymbol{D}}
\newcommand\bA{\boldsymbol{A}}
\newcommand\bB{\boldsymbol{B}}

\newcommand\bC{\boldsymbol{C}}
\newcommand\boldI{\boldsymbol{I}}
\newcommand\bGamma{\boldsymbol{\Gamma}}
\newcommand\IY{\mathbb{Y}}
\newcommand\IX{\mathbb{X}}
\newcommand\ID{\mathbb{D}}
\newcommand\IC{\mathbb{C}}

\makeatletter
\newcommand{\eqnum}{\refstepcounter{equation}\textup{\tagform@{\theequation}}}
\newcommand{\algo}{\textsc{TC-FISTA}\xspace}
\newcommand*{\transpose}{%
	{\mathpalette\@transpose{}}%
}
\newcommand*{\@transpose}[2]{%
	\raisebox{\depth}{$\m@th#1\intercal$}%
}

\makeatletter
\newcommand{\ostar}{\mathbin{\mathpalette\make@circled\star}}
\newcommand{\make@circled}[2]{%
	\ooalign{$\m@th#1\smallbigcirc{#1}$\cr\hidewidth$\m@th#1#2$\hidewidth\cr}%
}
\newcommand{\smallbigcirc}[1]{%
	\vcenter{\hbox{\scalebox{0.77778}{$\m@th#1\bigcirc$}}}%
}
\makeatother

\title{Tensor Convolutional Sparse Coding with Low-Rank activations, an application to EEG analysis}

\author{%
  Pierre~Humbert$^{1}$ \quad Laurent~Oudre$^{2}$ \quad Nivolas~Vayatis$^{1}$ \quad Julien~Audiffren$^{3}$ \vspace{.5em} \\
  $^{1}$ Université Paris-Saclay, CNRS, ENS Paris-Saclay, Centre Borelli, \\ F-91190 Gif-sur-Yvette, France \\
  $^{2}$ Université Sorbonne Paris Nord, L2TI, UR 4443, \\ F-93430 Villetaneuse, France \\
  $^{3}$ eXascale Infolab, University of Fribourg, Fribourg, Switzerland
}

\begin{document}
\maketitle
\begin{abstract}
\noindent
Recently, there has been growing interest in the analysis of spectrograms of ElectroEncephaloGram (EEG), particularly to study the neural correlates of (un)-consciousness during General Anesthesia (GA). Indeed, it has been shown that order three tensors (channels $\times$ frequencies $\times$ times) are a natural and useful representation of these signals.
However this encoding entails significant difficulties, especially for convolutional sparse coding (CSC) as existing methods do not take advantage of the particularities of tensor representation, such as rank structures, and are vulnerable to the high level of noise and perturbations that are inherent to EEG during medical acts.
To address this issue, in this paper we introduce a new CSC model, named Kruskal CSC (K-CSC), that uses the  Kruskal decomposition of the activation tensors to leverage the intrinsic low rank nature of these representations in order to extract relevant and interpretable encodings.
Our main contribution, \algo, uses multiple tools to efficiently solve the resulting optimization problem despite the increasing complexity induced by the tensor representation.
We then evaluate \algo on both synthetic dataset and real EEG recorded during GA. The results show that \algo is robust to noise and perturbations, resulting in accurate, sparse and interpretable encoding of the signals.
\end{abstract}

\section{Introduction}
General anesthesia (GA) consists in a medically induced state of unconsciousness, through the use of different drugs such as inhalational hypnotic anesthetics.
It is a cornerstone of modern medicine, and is crucial for the realization of many medical and surgical procedures \citep{brown2010general}.
However, anesthesia may carry some important risks (e.g. cognitive dysfunction \citep{punjasawadwong2018processed}, postoperative delirium \citep{fritz2016intraoperative}). Consequently, a sustained and careful monitoring of the level of consciousness of the patient -- also referred to as the Depth of Anesthesia (DoA) -- is required.
As a direct measurement of the electrical brain activity, the Electroencephalograms (EEG) recordings remain the gold-standard to assess the DoA and have extensively been used to study the phenomenons occurring during anesthesia-induced loss of consciousness \citep{liu2016closed}.
For instance, \cite{purdon2013electroencephalogram} has shown that the evolution of $\alpha$ and $\delta$-waves (respectively in the 8-13 Hz and 1-3Hz ranges) are promising predictors of the recovery of consciousness after GA that have been induced with propofol.

However, there is still no consensus on the existence and nature of universal markers that would track DoA, especially during transition phases such as Loss and Recovery of Consciousness (LOC, and ROC respectively).
This is partially due to the  heavily influence of the choice of the anesthetic drug over the behavior of EEG \citep{purdon2015clinical}, and intrinsic technical limitations. 
For instance, when data are recorded during real surgeries, the EEG signals are prone to low signal to noise ratio, impulsive noise due to sensor malfunctions, and artifacts caused by e.g. electro-surgical devices that are used to cut and cauterize tissue \citep{dubost2019selection}.
Nevertheless, there has been growing interest in the automatic analysis of EEG patterns in such settings.
In particular, we can distinguish two families of statistical learning methods that are used in this domain.

In tensor based approaches, EEG are frequently analysed by computing a Short-Time Fourier Transform (STFT) for each channel of the EEG, resulting in a tensor of order 3 encoding a Space-Time-Frequency (STF) representation (see e.g. \citep{miwakeichi2004decomposing, morup2006parallel, becker2010multiway, becker2014eeg, becker2015brain, zhao2011multilinear}).
The resulting tensor is then studied through the prism of the CANDECOMP/PARAFAC (CP) decomposition \citep{becker2014eeg}, in order to exploit the interactions among multiple modes.
This approach is frequently paired with the addition of a low CP-rank constraint on the tensor decomposition, resulting in representations that are more robust to noise and easier to interpret \citep{zhou2013tensor, zhao2011multilinear, cong2015tensor, rabusseau2016low}.

Similarly, there have been several attempts to extract new representations of unfiltered electrophysiological signals using shift-invariant dictionary learning approaches such as Convolutional Sparse Coding (CSC) (see e.g. \citep{jas2017learning, la2018multivariate}).
While these methods exhibit interesting results, their models are centered around atoms of order one, i.e. atoms that vary along a single mode.
Indeed, CSC methods are mainly focused on resolution for univariate signals or images \citep{garcia2018convolutional}, and therefore do not take into account the possible interaction between all the different modes of multivariate data -- a key property of the EEG spectrogram representation \citep{cong2015tensor}.
Moreover, these methods are frequently vulnerable to noise \citep{jas2017learning, simon2019rethinking} and structured perturbations such as impulsive noise \citep{9049336}.

Following these remarks, in this work we aim at combining the two previous approaches by introducing a Tensor CSC model, where both activations and atoms can be high order tensors. Inspired by recent works from \cite{phan2015low} and \cite{humbert2020low},  we constrain the solutions of this model to have low CP-rank activations, with the aim of obtaining accurate and interpretable decompositions of multidimensional data.

\paragraph{Contributions.}  In this article, we develop a low rank tensor based CSC model, named Kruskal CSC (or K-CSC for short), that is particularly adapted to analyse noisy tensor signals such as EEG spectrograms recorded during GA (Section \ref{sec:CDL}).
We propose an efficient optimization strategy to learn the different components of this model, named \algo,  that entails a low complexity and computation time (Section \ref{sec: model estimation}).
Using synthetic data, we show that this approach, where the rank constraint is enforced on activations instead of atoms, is significantly more robust to noise than its unconstrained counterparts. 
Finally, we evaluate our method on a dataset of EEG recording during GA. We show that \algo successfully isolates the complex noise and artifacts of the signal and produces a sparse, accurate and interpretable decomposition (Section \ref{sec:xp}).

\paragraph{Notation.} 
In the rest of this paper, tensors (resp. matrices, vectors) are represented with capital, bold calligraphic letters (resp. capital bold letters, bold lower letters).
$\calY_1, \ldots, \calY_N$ denote $N$  tensor-valued signals  in $\IY\triangleq\IR^{n_1 \times \cdots \times n_p}$,
$\{\calD_k\}_{1\le k\le K}$ are the $K$ multivariate atoms in $\ID \triangleq \IR^{w_1 \times \cdots \times w_p},$ with $(w_1 \leq n_1, \cdots, w_p \leq n_p)$.
$\{\calZ_{n, k}\}_{1\le k\le K, 1 \le n \le N} \in \IY$ are (multidimensional) activation maps.
Regarding tensor operations, $\ostar$ denotes the (circular) convolutional operator \citep{dudgeon1983multidimensional},
$\widehat{\cdot}$ denotes the frequency representation of a signal via the Discrete Fourier Transform (DFT),
$\circ$ represents the outer product,
$\otimes$ is the Kronecker product, 
$\odot$ is the Khatri-Rao product, 
and $\texttt{vec}$ denotes the vectorization transformation.
Finally,  we will use the word \textit{rank} to refer to the notion of \textit{Canonical Polyadic rank} (CP-rank) of tensors \citep{zhou2013tensor}.

\section{Low rank Tensor CSC}\label{sec:CDL}

In this section, we briefly define the multivariate CSC setting, before introducing K-CSC, the model used in this work. Then we discuss the ideas behind this model and its relation with the existing literature.

\paragraph {Multivariate CSC.}
Formally, given a finite set of $N$ multivariate signals $\calY_1, \ldots, \calY_N \in \IY$, tensors inputs of order $p > 0$ and $\balpha>0$  a scalar, the multivariate CSC problem is defined as
\begin{align}
\label{eq:M_CSC}
\min_{ \{\calD_k \in \ID, \calZ_{n, k}  \in \IY \}}
&\dfrac{1}{2}\sum_{n=1}^{N}\Bigg(\bigg\lVert \calY_{n} - \sum_{k=1}^{K} \calD_k \ostar \calZ_{n, k} \bigg\rVert_F^2
+ \balpha \sum^K_{k=1} \lVert \calZ_{n, k}\rVert_{1}  \Bigg) \quad \text{s.t.} \quad \| \calD_k\|_F \le 1 \quad \forall k \; .
\end{align}

In this formulation, the $\calZ_{n, k}$ are  activation maps which specify where the atoms $\calD_k$ are placed in the input signals. Note that both $\calZ_{n, k}$ and $\calD_k$ are tensors or order $p$. The constraint on $\calD_k$ prevents the scaling indeterminacy problem between  atoms and activations, as in the standard CSC (see e.g. \citep{bristow2013fast}).
The convolutional sparse coding of multi-channel signals has received only limited attention in the past \citep{barthelemy2012shift, wohlberg2016convolutional, garcia2018convolutional}, and only few existing work consider the application of convolutional sparse representations to signals with more than three channels \citep{garcia2018convolutional2}.

\paragraph {Low rank Tensor CSC.}
The key idea behind K-CSC is the use of a low rank constraint on the \textit{activation maps} in order to take advantage of the tensor structure. 
Indeed, previous works have shown that adding a low rank constraint to tensor learning problems significantly improves the performance of many methods (see e.g. \citep{zhou2013tensor, zhao2011multilinear, rabusseau2016low}.
To enforce this constraint in CSC, we  used the Kruskal operator, whose definition is recalled below

\begin{definition}\emph{(Kruskal operator) --}
	\label{def:kruskal_op}
Let $R>0$ and $p>0$. Then the \textit{Kruskal operator} of rank $R$ and of order $p$ denoted $[\![ \; \cdot \; ]\!]_{R,p}$  is defined as
\begin{align*}
    [\![ \; \cdot \; ]\!]_{R,p}:  \quad \IR^{n_1\times R} \times \ldots \times  \IR^{n_p\times R} \mapsto \IY \;, \qquad 
    [\![\bX^{(1)}, \cdots, \bX^{(p)}]\!]_{R,p} &\rightarrow \sum_{r=1}^{R} \bx^{(1)}_{r} \circ \cdots \circ \bx^{(p)}_{r} \; ,
\end{align*}
	where  $\forall 1\le i \le p,$  $\forall 1 \le r \le R,$  $\bx^{(i)}_{r}  \in \IR^{n_i}$  is the $r$-th column of  $\bX^{(i)}.$
\end{definition}

In the following, and to make notation easier, we use $[\![ \; \cdot \; ]\!]$ to denote $[\![ \; \cdot \; ]\!]_{R,p}$ when the values of $R,p$ are clear from the context.   Note that $[\![ \; \cdot \; ]\!]_{R,p}$ takes value in the subspace of tensors of rank lower than $R,$ and conversely any tensor in $\IY$ of rank lower than $R$ can be written using the Kruskal operator (see e.g. \citep{kruskal1989rank}). However this decomposition is not necessarily unique. 

Using the Kruskal operator, the Low rank Tensor CSC optimization problem can be rewritten as:
\begin{align}
\label{eq:K_CDL}
\min_{ \{\calD_k \calZ_{n, k}\}}
&\dfrac{1}{2}\sum_{n=1}^{N}\Bigg(\bigg\lVert \calY_{n} - \sum_{k=1}^{K} \calD_k \ostar \calZ_{n, k} \bigg\rVert_F^2
+\underbrace{\sum^p_{q=1} \balpha_{q} \sum^K_{k=1} \lVert \bZ^{(q)}_{n, k}\rVert_{1}}_{\textit{Mode Sparsity}} + \underbrace{\sum^p_{q=1} \bbeta_{q} \sum^K_{k=1} \lVert \bZ^{(q)}_{n, k}\rVert_F }_{\textit{Identifiability Regularization}}\Bigg)
\end{align}
\begin{equation*}
\text{s.t.} \quad \left\lbrace
\begin{aligned}
& \calZ_{n, k} = [\![\bZ^{(1)}_{n, k}, \cdots, \bZ^{(p)}_{n, k} ]\!]_{R,p} \quad \forall n, k \quad{\textit{(Rank Constraint)}}\\
& \calD_k \in \ID, \| \calD_k\|_F \le 1 \quad \forall k \; .
\end{aligned}
\right.
\end{equation*}

Compared to \eqref{eq:M_CSC},  this optimization problem presents the following crucial differences.
\begin{itemize}[leftmargin=*]
    \item \textbf{Rank constraint on the activations.} Due to this constraint, the rank of $\calZ_{n, k}$ is lower than $R.$ This aims at controlling the linear relation between the different modes of the activations maps and thus taking into account the natural structure of the data.
    While using the Kruskal decomposition to enforce a rank constraint on the dictionary has been studied (see e.g. \citep{rigamonti2013learning}), applying this constraint on the activations has received little attention \citep{humbert2020low}.
    The interests of this approach and its relation with the literature are discussed below.
    \item \textbf{\textbf{Identifiability and the Frobenius regularization.}}  The Kruskal decomposition, i.e. the tensor decomposition using the Kruskal operator, is prone to scaling indeterminacy\footnote{Equally valid Kruskal decompositions of a tensor can be obtained by multiplying two modes by respectively $\gamma$ and $1/\gamma$ for any $\gamma >0.$}, resulting in a continuous manifold of equivalent solutions.
    The Frobenius regularization in \eqref{eq:K_CDL} is known to addresses this problem (see \citep{paatero1997weighted} and \citep{acar2011scalable} for more details).
    \item \textbf{\textbf{Mode sparsity regularization.}} The  $\ell_1$-norms regularization on the modes $\bZ^{(q)}_{n, k}$ of the activation tensors enforces sparse solutions. However, contrary to the usual sparsity constraint in \eqref{eq:M_CSC}, we take advantage of the Kruskal decomposition and enforce different degree of sparsity for each mode of the activation tensors independently.  The sparsity in each mode is therefore controlled independently, allowing a higher flexibility -- particularly in applications where modes have different meaning, such as in EEG analysis where the sparsity constraint on channels is significantly lower than the one on time (see Section \ref{sec:xp}).
\end{itemize}

\paragraph{Low-rank constraints on sparse activations.}
The idea of enforcing low-rank constraints for CDL is not novel: \cite{rigamonti2013learning} and \cite{sironi2014learning} used the idea of separable filters for learning low-rank atoms in order to improve computational runtime. More recent publications including \cite{quesada2018separable, silva2018efficient, quesada2019combinatorial} have also successfully used low-rank constraints on dictionary to introduce efficient solvers but only in $2$-D. Importantly, in all the aforementioned works, the low-rank constraints is enforced on the atoms and not on the activations. Nevertheless, we claim that constraining the activations to be low rank brings two majors advantages.
First, in multiple application contexts the low-rank structure naturally appears in the activations rather than in the atoms/dictionary (see \citep{humbert2020low}, and Section \ref{sec:xp}). Our model allows to leverage this prior knowledge. 
Second, low-rank constraints on activation  entail a better robustness with respect to noise, which is one of the main weakness of the activation learning part of CDL \citep{simon2019rethinking}. Finally, the sparsity of the activations is supported by recent results on neuroscience which postulate that neural activity consists more of transient bursts of isolated events rather than rhythmically sustained oscillations \cite{van2018neural, karvat2020real}. Hence, such activities could be described not only by their frequency and amplitude but also by their rate, duration, and shape suggesting that multivariate CDL is well-adapted to analyze them.
\section{Resolution of the problem: \algo}\label{sec: model estimation}

We now introduce \algo, an algorithm that solves the K-CDL problem by iterating over a  dedicated FISTA solver.
\algo minimizes the objective function \eqref{eq:K_CDL} by using a bloc-coordinate strategy approach  \citep{hildreth1957quadratic, nikolova2017alternating}. 
Indeed, \algo splits the main non-convex problem into several convex subproblems,  by either  1) freezing all the atoms $\calD_k$  and all the modes of $\calZ$ except one (referred as $\calZ$-step, see Section \ref{subsec: z-step})  or  2) freezing $\calZ$ referred as $\calD$-step.
It is important to note that while \algo is not guaranteed to converge toward the global minimum (similarly to other CDL solvers, see \citep{bristow2013fast}), in our empirical evaluations the algorithm almost always converges to a near-optimal local minimum.
Since the optimization problem pertaining to  $\calD$-step is identical to the $\calD$-step of Multivariate CDL \eqref{eq:M_CSC}, it can be solved using any existing methods (see \citep{garcia2018convolutional} for a review). Following the recent result of \citep{wohlberg2015efficient}, we chose to use ADMM with iterative application of the Sherman-Morrison equation for the $\calD$-step of \algo. In our setting, this is by far the fastest method as reported in \citep{wohlberg2015efficient} and \citep{garcia2018convolutional}. The rest of this section is devoted to the $\calZ$-step.

\subsection*{$\calZ$-step : Activations update}\label{subsec: z-step}

In the remainder of the section, we consider that the atoms $\calD_k$ are fixed. Additionally, and to alleviate the notation,  we only consider the K-CSC problem with $N=1$ signal, and denote $\calZ_q \doteq \calZ_{1,q}$ and $\bZ^{(q)}_{1,k} \doteq \bZ^{(q)}_{k} $. All the results and methods can be immediately extended to any value of $N>1$ as the activation maps are independent across the signals $\calY_1, \ldots, \calY_N$.

\paragraph{Block-Coordinate Descent.} One of the main difficulties of the K-CDL model is that because of the low rank constraint,  \eqref{eq:K_CDL} is no longer convex with respect to $\calZ$.
While it is theoretically possible to solve the optimization problem by rewriting it as a regression problem and using methods proposed by \citep{zhou2013tensor, he2018boosted}, this approach would require the construction and manipulation of a circulant tensor of size $ \approx K \times (n_1 \times \ldots \times n_p)^2$ which is not tractable in practice.
To address this problem, we consider the modes of $\calZ$, defined as $(\bZ^{(q)}_{1}, \cdots, \bZ^{(q)}_{K}),$ for any value of  $q\in[1,\ldots,p]$. Indeed, \eqref{eq:K_CDL} is convex with respect to each of the modes of $\calZ$ individually, i.e. when all the other modes are fixed. 
Furthermore, the two regularization terms of \eqref{eq:K_CDL} are separable with respect to the $\bZ^{(q)}_{k}$.
Building on this remark, \algo does not directly solve \eqref{eq:K_CDL} but instead uses a block-coordinate strategy,  minimizing $p$ different sub-problems where all \textit{modes} except the $q$-th one ($1 \le q \le p$) of each activation tensor are fixed.
The resulting optimization sub-problems can be rewritten as
\begin{equation}
\label{eq:z1_step}
\min_{\bZ_{1}^{(q)}, \cdots, \bZ_{K}^{(q)}} \;  \dfrac{1}{2} \underbrace{\Bigg\lVert \calY - \sum_{k=1}^{K} \calD_k \ostar_{1, \cdots, p} [\![\bZ_{k}^{(1)}, \cdots, \bZ_{k}^{(p)}]\!] \Bigg\rVert_F^2}_{f\left(\{\bZ_{k}^{(r)}\}^{K, p}_{k=1, r=1}\right)}  + \underbrace{\balpha_q\sum_{k=1}^{K} \lVert \bZ_{k}^{(q)}\rVert_1 + \bbeta_q \sum_{k=1}^{K} \lVert \bZ_{k}^{(q)}\rVert^2_F}_{g\left(\{\bZ_{k}^{(q)}\}^K_{k=1}\right)} \; , 
\end{equation}
where $f$ is \textit{the fidelity term} which controls the difference between the input and its reconstruction, and $g$ is the sum of the regularization terms.

\paragraph{Solving the block-coordinate problem.}
In standard CSC, two solvers dominate the literature to learn the activation maps: Alternating Direction Method of Multipliers (ADMM) \citep{glowinski1975approximation} and Fast Iterative Soft Thresholding Algorithm (FISTA) \citep{beck2009fast}. Both are competitive and recent comparative reviews \citep{garcia2018convolutional} indicate a very wide range of performances across existing methods. In this article, we chose to focus on FISTA to solve \eqref{eq:z1_step}, but our approach can also be used with the ADMM solver.

Solving \eqref{eq:z1_step} with FISTA requires repeated  A)  computation of the gradient of $f$ and B)  evaluation of the proximal operator of $g$.
For the later, as $g$ is fully separable, its proximal is given by the weighted soft-thresholding operator for each $\bZ_{k, q}$ \citep{parikh2014proximal}. Notice that any $g$ with known proximal can be used with \algo. For instance, in Section \ref{sec:xp}, we add a  non-negative constaint on the activations  when using spectrograms derived from EEG, resulting in a proximal given by the soft non-negative thresholding operator. Regarding A), calculating the gradient of $f$ is computationally demanding, particularly in the tensor setting, and the rest of this section is devoted to introduce several important  optimizations.

\paragraph{Accelerating gradient computation using Fourier transform.}
Recent works \citep{wohlberg2016convolutional, garcia2018convolutional} have shown that  projecting the computation of the gradient of $f$ in the Fourier domain drastically decreases its complexity, resulting in a significantly faster optimization. 
This method has proven to be important to perform CSC quickly and efficiently on complex and large signals, and is crucial to the tensor-based CSC presented in this paper due to the exponential increase of the size of signals and atoms.
However, the generalization of this technique to our tensor setting \eqref{eq:z1_step} is not straightforward. In the following, we provide a new formulation of the gradient in the Fourier domain, that is easy to compute efficiently.

\begin{theorem}\emph{(Compact vectorized formulation) --}
	\label{thm:vect_form}
    	Let $\widehat{\by}^{(q)}$ (resp. $\widehat{\bd}_k^{(q)} $)  be the vectorization of the folding of $\widehat{\calY}$ (resp. $\widehat{\calD_k}$) along the dimension $q$,
    	$\widehat{\bZ}_{k}^{(q)} = [\widehat{\bZ}_{k}^{(q)}(:, 1) \mid \ldots \mid \widehat{\bZ}_{k}^{(q)}(:, R)]$ be the DFT of $\bZ_{k}^{(q)}$ along its columns and
    	$\widehat{\bz}^{(q)}_{k} = \texttt{vec}(\widehat{\bZ}_{k}^{(q)}).$ 
    	Then 
	\begin{equation}\label{eq:f vector thm}
	f\left(\{\bZ_{k}^{(r)}\}^{K, p}_{k=1, r=1}\right) \propto \dfrac{1}{2} \Big\lVert \widehat{\by}^{(q)} - \widehat{\bGamma}^{(q)} (\widehat{\bA}^{(q)} \otimes \boldI^{(q)})\widehat{\bz}^{(q)} \Big\rVert^2_F \; ,
	\end{equation}
	where 
	$\widehat{\bz}^{(q)} = [\widehat{\bz}^{(q)^\transpose}_1, \ldots, \widehat{\bz}^{(q)^\transpose}_K]^\transpose \in \IC^{K R n_q}$,
	$\widehat{\bGamma}^{(q)} = [\text{diag}(\widehat{\bd_1}^{(q)}), \ldots, \text{diag}(\widehat{\bd_K}^{(q)})] \in \IC^{n_1 \cdots n_p \times K n_1 \cdots n_p}$,
	$\boldI^{(q)} \in \IR^{n_q \times n_q}$ is the identity matrix
	 and 
	\begin{equation*}
	\widehat{\bA}^{(q)} = \begin{pmatrix} 
	\widehat{\bB}^{(q)}_{1} &  & 0 \\ 
	& \ddots &  \\ 
	0 &  & \widehat{\bB}^{(q)}_{K}
	\end{pmatrix}  \in \IC^{K \prod_{1, i \neq q}^{p} n_i \times K R} \quad \text{with} \quad \widehat{\bB}^{(q)}_{k} = (\stackrel{\hookleftarrow}{\odot}^p_{i=1, i\neq q} \widehat{\bZ}_{k}^{(i)} ) \; .
	\end{equation*}
\end{theorem}
\begin{proof}
Equation \eqref{eq:f vector thm} is derived by using Plancherel theorem and successive tensor manipulations. The full proof can be found in the appendix.
\end{proof}
From Theorem \ref{thm:vect_form}, we can derive the gradient of $f$ in the Fourier domain.
\begin{corollary}\emph{(Gradient of $f$) --} With the notation of Theorem \ref{thm:vect_form}, the gradient of $f$ with respect to $\bz^{(q)} = [ \texttt{vec}({\bZ}^{(q)}_1)^\transpose, \ldots, \texttt{vec}({\bZ}^{(q)}_K)^\transpose]^\transpose$ is given by
\begin{equation}\label{eq:gradient formula}
\nabla_{\bz^{(q)}} 	f\left(\{\bZ_{k}^{(r)}\}^{K, p}_{k=1, r=1}\right) =  \text{IDFT}\left[\left((\widehat{\bA}^{(q)} \otimes \boldI)\widehat{\bGamma}^{(q)}\right)^H\left(\widehat{\bGamma}^{(q)} (\widehat{\bA}^{(q)} \otimes \boldI)\widehat{\bz}^{(q) }- \widehat{\by}^{(q)}\right)\right] \; ,
\end{equation}
where IDFT$[\cdot]$ stands for the Inverse Discrete Fourier Transform.
\end{corollary}
Equation \eqref{eq:gradient formula} provides a closed-form equation for computing the values of the partial derivative of $f$. Crucially, some of its terms can be precomputed, further improving the rapidity of the $\calZ$-step. This point is discussed in the next paragraph. In the following, we drop the upper-script $\cdot^{(q)}$ of $\widehat{\bA}^{(q)}$ and $\widehat{\bGamma}^{(q)}$ for clarity.

\paragraph{Efficient computation of the gradient via the Gram matrix.} 
In the case where $\prod_{i=1}^{p} n_i \gg KRn_q$ (low rank assumption), two terms of \eqref{eq:gradient formula} can be precomputed efficiently:  $(\widehat{\bA}^{H} \otimes \boldI)\widehat{\bGamma}^{H}\widehat{\by}^{(q)}$ and the Gram matrix $\bG \triangleq (\widehat{\bA}^H \otimes \boldI)\widehat{\bGamma}^H\widehat{\bGamma} (\widehat{\bA} \otimes \boldI).$
For the first term, we have
\begin{equation*}
	(\widehat{\bA} \otimes \boldI)\widehat{\by}^{(q)} = (\widehat{\bA} \otimes \boldI)\texttt{vec}( \bY^{(q)} ) = \texttt{vec}(\bY^{(q)}\widehat{\bA}^\transpose) \; ,
\end{equation*}
whose computation can be performed in $\calO(KR \prod_{i=1}^{p}n_i)$ operations,  instead of $\calO(KR n_q \prod_{i=1}^{p}n_i)$ for a naive strategy.
Moreover, the following proposition shows that the Gram matrix $\bG$ can be computed independently and efficiently.
\begin{proposition}\label{prop:complexity gram}
The matrix $\bG \triangleq (\widehat{\bA}^H \otimes \boldI)\widehat{\bGamma}^H\widehat{\bGamma} (\widehat{\bA} \otimes \boldI)$ can be obtained by computing $K^2$ blocks $\bG_{k,\ell}, 1 \le k,\ell \le K,$  where
\begin{equation}
 \bG_{k,\ell} = \left((\stackrel{\hookleftarrow}{\odot}^p_{i=1, i\neq q} \widehat{\bZ}_{k, i} )^H \otimes I\right) \overline{\text{diag}(\widehat{\bd_k}^{(q)})} \text{diag}(\widehat{\bd_\ell}^{(q)}) \left((\stackrel{\hookleftarrow}{\odot}^p_{i=1, i\neq q} \widehat{\bZ}_{\ell, i} ) \otimes I\right) \; ,
\end{equation}
and each of these blocks can be computed in $\calO(R^2 \prod_{i=1, i\neq q}^{p}n_i)$ operations. 
\end{proposition}
As a consequence of proposition \ref{prop:complexity gram}, the full matrix $\bG$ can be computed in $\calO((KR)^2 \prod_{i=1, i\neq q}^{p}n_i)$ operations. Moreover, $\bG$ is a particular $(KRn_q \times KRn_q)$ banded matrix (see appendix), and consequently its product with $\widehat{\bz}^{(q)}$ can be made in only $\calO((KR)^2 n_q)$ operations instead of $\calO((KRn_q)^2)$ with a naive strategy. The resulting total complexity of the $\calZ$-step of \algo and classical CSC approaches are reported in Table \ref{tab:complexity_csct}.

\begin{table}
\caption{Time and Space complexities of one iteration of the $\calZ$-step for  \algo, \textbf{FCSC-ShM} with Sherman-Morrison iterations \citep{bristow2013fast, wohlberg2015efficient} and \textbf{ConvFISTA-FD} in the Fourier domain \citep{wohlberg2015efficient, garcia2018convolutional}. $K$ denotes the number of atoms, $n_i$ is the size of the signal along order $i$, $M = \prod_{i=1}^p n_i$ is the size of the signal and $R$ is the maximal rank.}
\label{tab:complexity_csct}
	\footnotesize
	\centering
	\ra{1.2}
\begin{adjustbox}{max width=\textwidth}
	\begin{tabular}{@{}llll@{}}
		\toprule
		\textbf{Algorithm} & \textbf{Time complexity ($\calZ$-step)} & \textbf{Dominant term } & \textbf{Space} \\ \midrule
		\textbf{FCSC-ShM} & $\underbrace{K M}_{\text{Linear systems}} + \underbrace{K M \log(M)}_{\text{FFTs}} + \underbrace{K M}_{\text{Shrinkage}}$ & $K \left(\prod^p_{q=1} n_q\right) \log(\prod^p_{q=1} n_q)$ & $K \left(\prod^p_{q=1} n_q\right)$ \\
		\textbf{ConvFISTA-FD} & $\underbrace{K M}_{\text{Gradient}} + \underbrace{K M \log(M)}_{\text{FFTs}} + \underbrace{K M}_{\text{Shrinkage}}$ & $K \left(\prod^p_{q=1} n_q\right) \log(\prod^p_{q=1} n_q)$ & $K \left(\prod^p_{q=1} n_q\right)$ \\
		\textbf{T-ConvFISTA} & $\underbrace{(KR)^2 \sum^p_{q=1} n_q}_{\text{Gradient}} + \underbrace{KR \sum^p_{q=1} n_q \log(M)}_{\text{FFTs}} + \underbrace{KR \sum^p_{q=1} n_q}_{\text{Shrinkage}}$ & $(KR)^2 \sum^p_{q=1} n_q$ & $KR \max(n_q)$\\
		\bottomrule
	\end{tabular}
\end{adjustbox}
\end{table}
\section{Experiments}\label{sec:xp}

All numerical experiments were run on a Linux laptop with 12-core 2.5GHz Intel CPUs using standard \texttt{Python} libraries, \texttt{Tensorly} \citep{kossaifi2019tensorly}, \texttt{Sporco} \citep{wohlberg-2017-sporco}, and all the speed-ups from Section \ref{subsec: z-step}. Extended results and additional experiments can be found in the appendix. Our code is publicly available at \url{https://github.com/pierreHmbt/Tensor_CDL}.

\subsection{Evaluation on synthetic data}\label{sec:xp-synt}
We first evaluated the performance of \algo on a synthetic dataset, in order to illustrate the advantages of the low rank approach and highlight the benefits of the gradient computation optimization proposed in Section \ref{sec: model estimation}.
In these experiments, \algo was compared with the two state-of-the art methods for multivariate CSC, \textbf{FCSC} with Sherman-Morrison iterations (FCSC-ShM) \citep{bristow2013fast, wohlberg2015efficient} and \textbf{ConvFISTA} in the Fourier domain (ConvFISTA-FD) \citep{wohlberg2015efficient, garcia2018convolutional}. To allow for fair speed comparison, we used the strategy proposed in \citep{mairal2009online, mairal2010online}, and re-implemented the other two methods in  \texttt{Python}.
\paragraph{Synthetic simulation setup.} Experiments were conducted on a random synthetic dataset, containing 10 independent signals of size   $128 \times 128 \times 128$  generated using  $K=3$ atoms of size $5 \times 5 \times 5$ and  activation tensors of rank $R^* \le 2.$
Following \cite{wohlberg2015efficient},  we added a multivariate Gaussian noise to each signal to obtain a low level of noise (SNR of $25$dB) or a high level of noise (SNR of $10$dB). Each hyperparameter was chosen between $[0.0001, 100]$ for all methods. The gradient step of \algo was set to the inverse of the Lipschitz constant of each subproblem \eqref{eq:f vector thm}.

\paragraph{\algo with known dictionary.}
First, we evaluated the performance of the different algorithms   on the $\calZ-$step -- i.e. the reconstruction of the activation tensors using a known dictionary -- for different values of the noise ratio, and  rank constraint $R$ (for \algo). Each method was run with five different random initializations, and the one returning the best result was kept. We measured the performance using the Root Mean Square Error (RMSE) between the true values of the activations $\calZ$ and their estimated values. The results are reported in Figure \ref{fig:rmse_csc_comparison} (left). First, note that if the rank is underestimated ($R=1$), then \algo does not learn a suitable estimation of the activations. However, interestingly, if the rank is overestimated ($R=3$ or $4$), then \algo only suffers from a small drop in performance. These two behavior are in line with previous results on low rank tensors estimation \citep{zhou2013tensor}. We also see that \algo (with $R>1$) performs significantly better than its counterparts without rank constraints, particularly in presence of strong noise, justifying the low rank activation approach. Finally, note that despite using different initializations, \algo consistently learns a good estimator of $\calZ$, resulting in a low standard deviation.

\paragraph{\algo with unknown dictionary.} 
We then compared the performance of the aforementioned algorithms on the complete K-CSC problem, which includes the simultaneous learning of the activations ($\calZ$-step) and the atoms ($\calD$-step), on signals with high level of noise.
Figure \ref{fig:rmse_csc_comparison}(right) shows the average time until convergence (i.e. until the relative convergence tolerance becomes lower than $1e-4$ \citep{boyd2011distributed}).
While it is important to remark that the relative speeds of each methods are dependent of their choice of hyperparameters as well as on the sparsity of the signals, we observe that A) \algo with the optimizations discussed in Section \ref{sec: model estimation} is significantly faster than its regular counterpart and B) that \algo is faster than FCSC-ShM and ConvFista, even if the advantage decreases as $R$ increases.
This is in line with the time complexity of each algorithm, as presented in Table \ref{tab:complexity_csct}.

\begin{figure}
    \centering
    \begin{minipage}[t]{.39\linewidth}
        \centering
        \includegraphics[width=\linewidth]{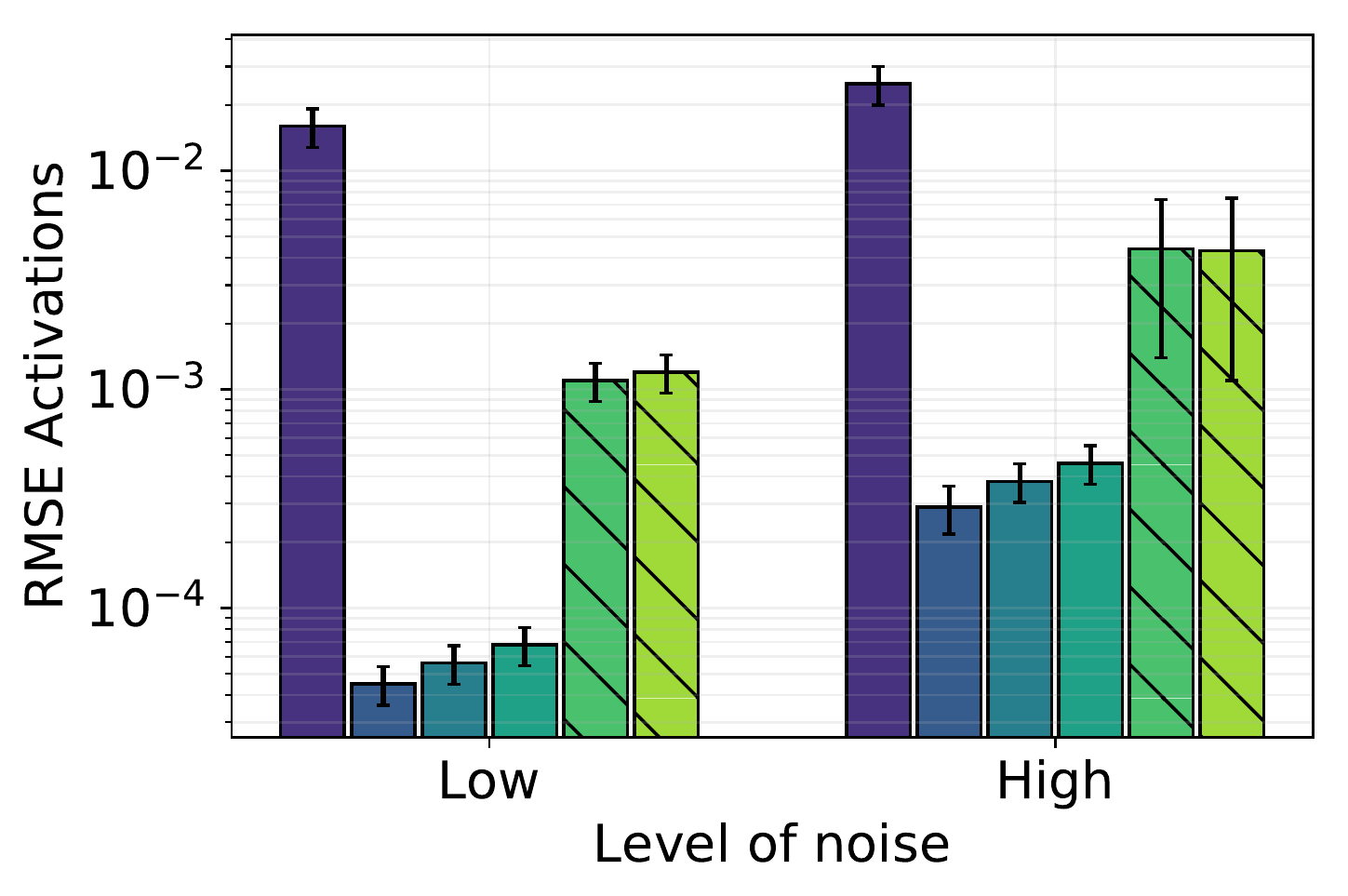}\\
        {\small (a)\;}
    \end{minipage}
    \begin{minipage}[t]{.39\linewidth}
        \centering
        \includegraphics[width=\linewidth]{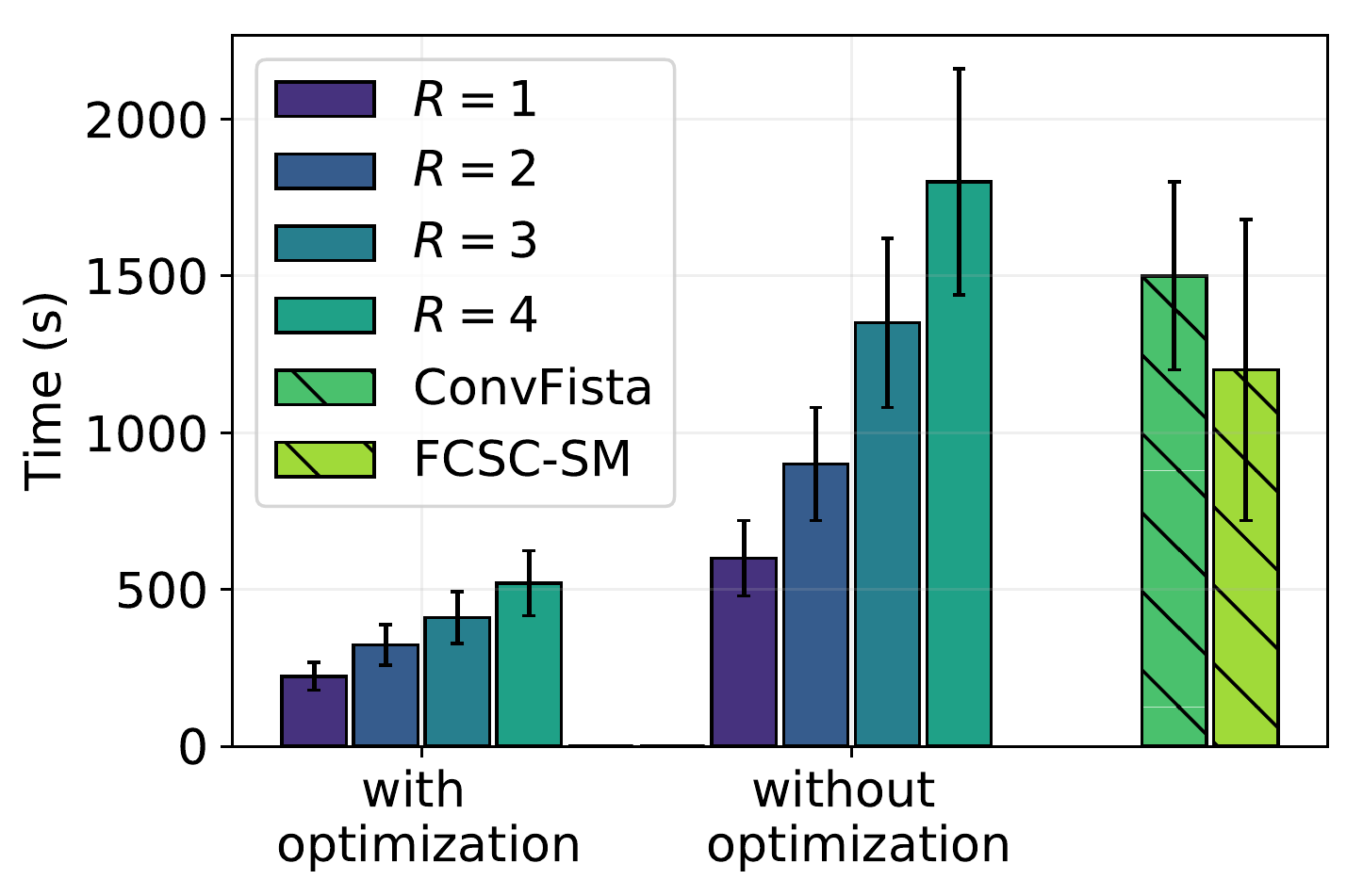}\\
        {\small (b)\;}
    \end{minipage}
    \caption{\textbf{Left}: Average RMSE (log-scale) between true and reconstructed activation tensors ($\calZ$-step) on the synthetic dataset. Low (resp. High) noise denotes a SNR of 25 (resp. 10). \textbf{Right}: Time until convergence of the dictionary learning process  ($\calZ$ + $\calD$ steps), with and without the optimizations discussed in Section \ref{sec: model estimation}. The standard deviation are indicated using black lines.}
    \label{fig:rmse_csc_comparison}
\end{figure}

\subsection{Multichannel EEG signals during a General Anesthesia}
In this experiment, we study multichannel EEG signals during a General Anesthesia induced by propofol and maintained by sevoflurane. As this setting is well-studied, the expected behavior of EEG is known, and could thus validate our approach in order to use it in other less studied configurations.

\paragraph{Dataset.} We now apply \algo to the real dataset of $32$-channel  EEG signals recorded at $250$ Hz during a General Anesthesia (GA) issued from \citep{dubost2019selection}. Compared to usual MEG datasets (see e.g. \citep{gramfort2013meg}), signals exhibit low SNR, impulsive noise and artefacts \citep{dubost2019selection}.  We crop the signal to focus on the Recovery of Consciousness (RoC), a crucial and difficult part of GA \cite{purdon2013electroencephalogram}, for a resulting duration of $1000$ seconds per signal. Following usual EEG preprocessing, the signals are filtered using a bandpass filter between $1$ and $20$Hz, and then transformed using a STFT with window size equals to $1024$ samples and $50$\% overlap (bins of size $0.17$Hz), resulting in a STF representation \citep{miwakeichi2004decomposing, morup2006parallel, becker2010multiway, becker2014eeg, becker2015brain, zhao2011multilinear}. We obtain a tensor $\calY$ of size $(32 \times 82 \times 490)$.
During a GA, patients are known to be static and EEG signals present a small number of patterns.
As a consequence, we set $R = 3$, and learn $K = 4$ atoms of size ($ 1 \times 15 \times 5$) corresponding to time-frequency atoms covering $8.19$ seconds and a band of frequencies of $3.42$Hz.
Hyperparameters are chosen similarly to the Section \ref{sec:xp-synt}. Finally, as previously mention, we add a non-negative constraint on the activations to fit the non-negativeness of spectrograms.

\paragraph{Results: $\boldsymbol{\alpha}$ and $\boldsymbol{\delta}$ waves.}
We present in Figure \ref{fig:eeg_atom_small} three atoms\footnote{In our experiments on this dataset, these atoms are consistently learnt by \algo, while the fourth atom varies and does not appear to encode any meaningful information.} learnt by \algo, together with their partial activations.
First, it is interesting to note that the learnt activations of the first two atoms  (resp. atom 3) are rank $2$ (resp. 1) instead of $3$ -- the actual constraint.
This illustrates the robustness of the \algo representation  with respect to mildly overestimated optimal rank-- which may be explained by the combination of rank constraint and sparsity inducing penalty.
Second, due to resulting low rank, these activations can be interpreted by looking at their behavior along the different modes (channel $\times$ frequency $\times$ time), as represented in  Figure \ref{fig:eeg_atom_small}. 
Hence, by studying the frequency activations (mode $2$), we see that the first two atoms encodes $\alpha$ and $\delta$-waves, two important frequencies that occur during a GA induced by propofol \cite{purdon2013electroencephalogram}.
Similarly, we observe that their time activations (mode $3$) are decreasing with
a behavior that was suggested by previous work during the RoC phase \citep{purdon2013electroencephalogram}.

\begin{figure}[t]
	\centering
	\includegraphics[width=.82\linewidth]{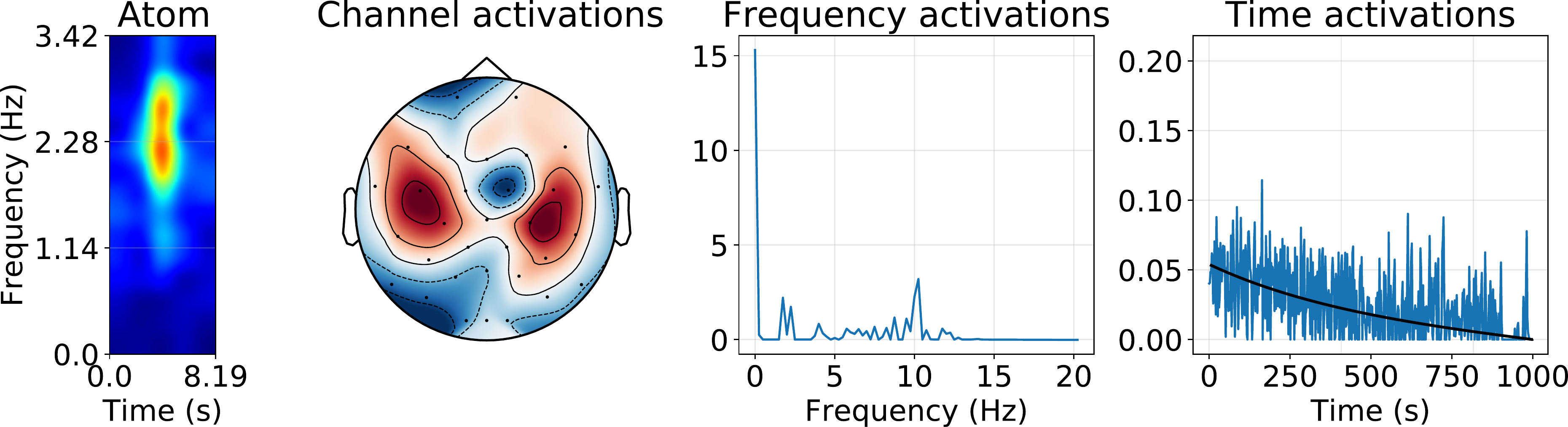}
	\includegraphics[width=.82\linewidth]{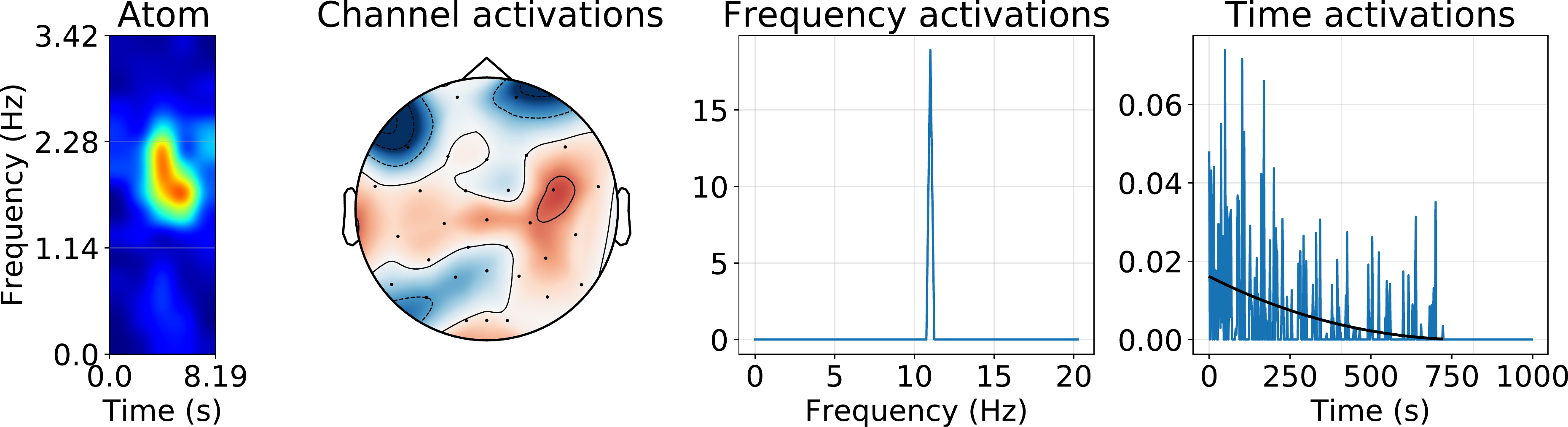}
	\includegraphics[width=.82\linewidth]{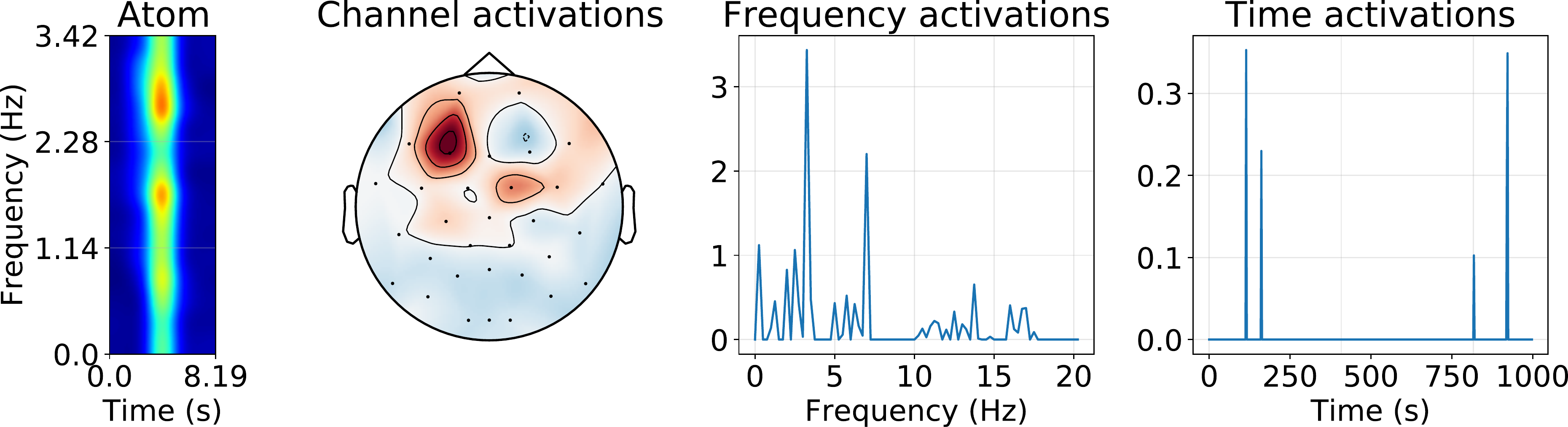}
	\caption{Three atoms of interest with their partial activations. From left to right: the time-frequency atom, the channel activations (mode $1$) on the scalp draw with MNE \citep{gramfort2014mne}, the frequency activations (mode $2$), and the time activations (mode $3$).}
	\label{fig:eeg_atom_small}
\end{figure}

\paragraph{Detecting impulsive noise.}
Interestingly, the third atom is of completely different nature, as it spans a wide range of frequencies. This atom may be seen as an encoding of impulsive noise, as its activations are mostly concentrated around one channel (mode $1$).
By using the activation of the third atom learnt by \algo, this impulsive noise can be removed from the final reconstruction by vanishing its contribution, highlighting the robust behavior of \algo.

\paragraph{Robustness to noise and reconstruction.}
By studying the channel activations (mode $1$) of atoms 1 and 2, we observe the presence of  three deficient channels: $10$ (CP1), $21$ (CP2), and $28$ (F4). This is coherent with clinical context, as these channels are located in the frontal and central area -- where individual anatomic variations can explain a contact defect -- resulting in sensors recording only noise at these positions. 
Interestingly, due to the low-rank constraint, our model  is robust to this problem  -- as it assumes  string correlation between the channels -- and is able to reconstruct the missing or bad signal. This property emphasizing the robust nature of \algo with respect to unstructured noise and sensor defect.
\section{Conclusion}
In this paper, we developed a tensor based CSC framework, named Kruskal CSC, where signals, atoms and activations are assumed to be tensors, but activations are constrained to have a low CP-rank. 
We proposed an efficient optimisation strategy to solve this framework, named \algo,  that entails a reduced complexity and lower computation time. 
We experimentally evaluated our method, and showed that our  \algo is significantly more robust to noise and artefacts than its unconstrained counterparts, and as a result is able to produce a sparse, accurate and interpretable decomposition of noisy tensor signals such as EEG spectrograms recorded during propofol induced GA. 
Future work may include the use of \algo to other types of GA, as well as EEG recordings, or other types of tensor based signals such as multi-source music.

\section*{Acknowledgement}
The authors would like to thank the Service d'Anesthésie-Réanimation of the Hôpital d'Instruction des Armées Bégin and especially Dr. Clément Dubost for the access to the data, the fruitful discussions, and interpretations of the recovered patterns.

\balance
\bibliographystyle{abbrvnat}
\bibliography{biblio}

\newpage
\appendix

\begin{center}
	\large
	APPENDIX
\end{center}

\section{Pseudo code and implementation}
The  pseudo-code of  the $\calZ$-step of \algo for one mode can be found in Algorithm \ref{algo:TCFISTA}. The implementation of \algo using Python is available in the file \texttt{ \url{code/TC_FISTA.py}}, together with notebooks that provide illustrations of the experiments.

\begin{algorithm}[H]
	\caption{\algo (sub-problem)}
	\begin{algorithmic}
        \State\textbf{Input:} signal $\calY$, dictionary $\calD_1, \cdots, \calD_K$, regularization and step parameters $\alpha, \beta$, $\eta$ ($\eta = 1/L$, the inverse of Lipschitz constant if calculate), tolerance $\varepsilon$
		\State\textbf{Initialization:  $\bZ^{(0)}$}
        \State \State\textbf{Precompute:} $\widehat{\calY}$, $\{\widehat{\bD}_k\}$, $\bG$ and $(\widehat{\bA} \otimes \boldI)\widehat{\by}^{(q)}$
        \vspace{1em}
		
		\State $t^{(0)} \longleftarrow 1$
		
		\Repeat
		\vspace{0.5em}
		\State \texttt{$\rhd$ Update of $\bW$ via a proximal gradient step (ISTA)}
		\State Compute $\widehat{\bZ}^{(s)}$ using the DFT
		\State $\widehat{\bz}^{(s)} \longleftarrow \texttt{vec}(\widehat{\bZ}^{(s)})$
		\State $\widehat{\bw}^{(s + 1/2)} \longleftarrow \widehat{\bz}^{(s)} - \eta \left(\bG  \widehat{\bz}^{(s)} - (\widehat{\bA} \otimes \boldI)\widehat{\by}^{(q)}\right)$
		\State $\widehat{\bW}^{(s + 1/2)} \longleftarrow$ Matricization of $\widehat{\bw}^{(s + 1/2)}$
		\State Compute $\bW^{(s + 1/2)}$ using the IDFT
		
		\State \texttt{$\rhd$ Update of $\bW$ via a proximal step (ISTA)}
		\State $\bW^{(s+1)} \longleftarrow \text{prox}_{\eta, \alpha, \beta} \left(\bW^{(s+1/2)}_k\right)$
		
		\State \texttt{$\rhd$ Nesterov momentum step (FISTA)}
		\State $t^{(s + 1)} \longleftarrow \dfrac{1 + \sqrt{1 + 4 \cdot t^{(s)^2}}}{2}$
		\State $\bZ^{(s+1)} \longleftarrow \bW^{(s+1)} + \dfrac{t^{(s)} - 1}{t^{(s+1)} + 1} (\bW^{(s+1)} - \bW^{(s)})$

		\Until{$\lVert \bZ^{(s+1)} - \bZ^{(s)}  \rVert_\infty \leq \varepsilon $}
	\end{algorithmic}
	\label{algo:TCFISTA}
	\normalsize
\end{algorithm}

\section{Additional results on synthetic data}

To further illustrate and compare the effectiveness and efficiency of \algo, we consider in this section one additional (smaller) dataset and other noise levels.
We also provide additional information on the data sets considered in the main paper and their generation.

\noindent
For the convenience of the reader, we list here the CSC/CDL algorithms compared and the acronyms we use throughout this section: FISTA with tensor-based rank constraint (\algo), ADMM with iterative application of the Sherman-Morrison equation (FCSC-ShM) \citep{bristow2013fast, wohlberg2015efficient}, FISTA in the Fourier domain (ConvFISTA-FD) \citep{chalasani2013fast, wohlberg2015efficient}.

\begin{figure}[H]
	\centering
    \includegraphics[width=.8\linewidth]{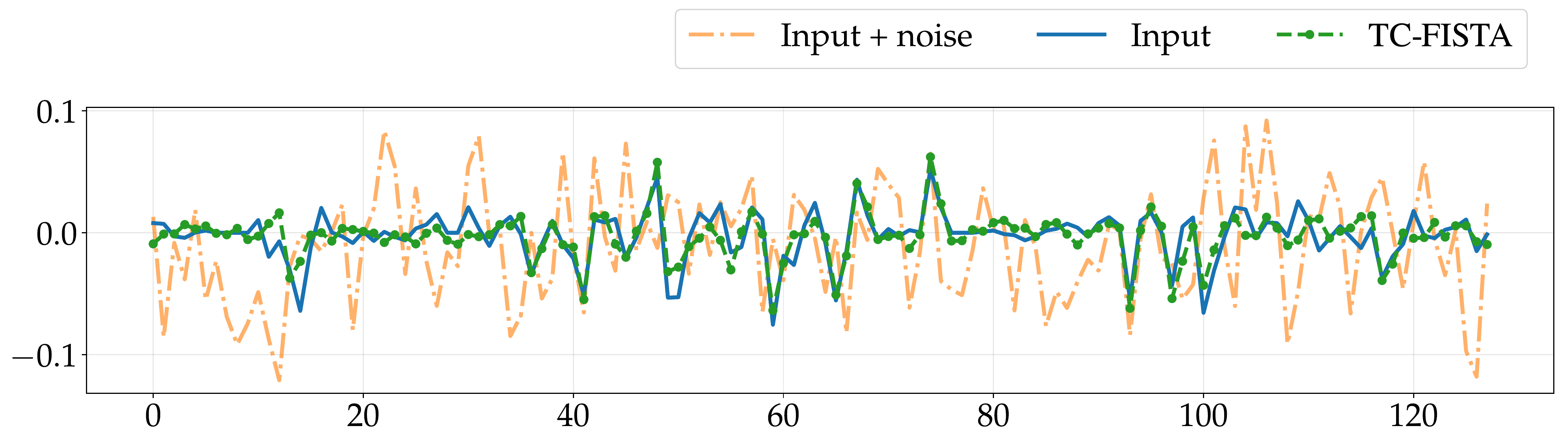}
    \includegraphics[width=.8\linewidth]{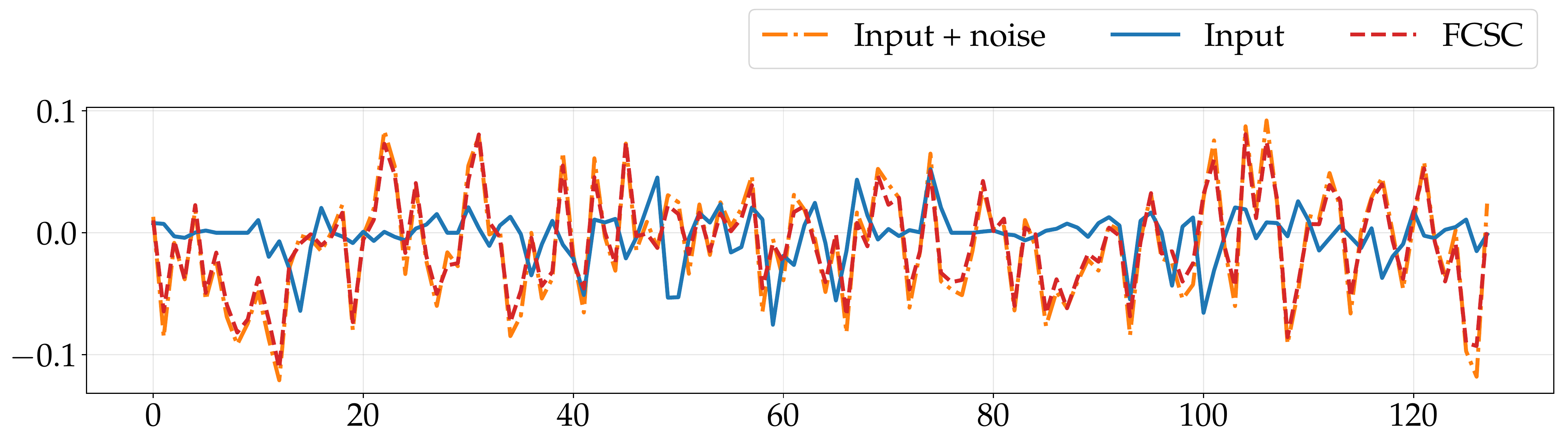}
	\caption{One tube of $3$-rd order (Top) input and reconstruction with \algo, and (Bottom) input + noise and reconstruction with FCSC-ShM. ConvFISTA-FD leads to a similar reconstruction.}
	\label{fig:csc_noisy_app}
\end{figure}

\subsection{Setup}
\paragraph{Dataset.} Small-scale and large-scale experiments are performed by considering two main different datasets

\begin{itemize}[leftmargin=*]
	\item A \textit{small-scale dataset} which contains $10$ independent input signals of size $(25 \times 25 \times 25)$. Each signal is generated as follows. We draw $K=3$ atoms of size $(5 \times 5 \times 5)$ according to an Uniform distribution with values in $[-1, 1]$ and normalize them. Then, we set the maximal CP-rank to $R^*\leq 2$ and drawn the sparse activations from a Bernoulli-Uniform distribution with Bernoulli parameter equals to $0.2$, and range of values in $[-1, 1]$. Finally, we generate the input tensor according to the convolutional model induced by the K-CDL model. Following \cite{wohlberg2015efficient},  we added a multivariate Gaussian noise to each signal to obtain a Signal to Nosie Ratio (SNR) of $25$dB, $10$dB, and $5$dB.
	\item A \textit{large-scale dataset} -- the one from the main paper -- which is generated as the small-scale dataset but with input signals of size $(128 \times 128 \times 128)$ and Bernoulli parameter equals to $0.02$.
\end{itemize}

Recall that the definition of the SNR between a signal $\by_{ref}$ and a comparison one $\by_{noisy} = \by_{ref} + \boldsymbol{\varepsilon}$ is
	\begin{equation*}
	 \text{SNR}(\by_{ref}, \by_{noisy}) = 10 \log_{10} \left(\dfrac{\text{Var}(\by_{ref})}{\text{MSE}(\by_{ref}, \by_{noisy})}\right) \; ,
	\end{equation*}
	 where MSE denotes the Mean Squared Error. SNR is an asymmetric decibel measurement (dB) used to compare the level of a signal to the level of background noise \citep{wohlberg2015efficient}.

\paragraph{Metrics.} We use four metrics to evaluate our methods
\begin{itemize}[leftmargin=*]
	\item The Root Mean Square Error (RMSE) between the true input signal (resp. the true activations) and the reconstruction (resp. the learnt activations). The lowest the better. We denote them RMSE$(\calY)$ and RMSE$(\calZ)$.
	\item The number of times, among all the initializations, our method reaches a ``correct'' minimizers, e.g. RMSE under $\varepsilon = 1. e{-6}$. This metric reflects the dependency of an algorithm to its initializations. The higher the better. We denoted them $\#\{$RMSE$(\calY) < \varepsilon\}$ and $\#\{$RMSE$(\calZ) < \varepsilon\}$ where $\varepsilon$ need to be set.
\end{itemize}

\subsection{Evaluation of the K-CSC}
On this additional experiment, we compare the performances of \algo on the CSC task and show the importance of the rank structure to get rid of the noise. The true dictionary is therefore given at the beginning. The $\{\bZ_{k, q}\}$ are initialized with random Uniform matrices.

\paragraph{Results with noise.}
For each input signal, we run \algo with five different initializations. The metric $\#\{$RMSE$(\cdot) < \varepsilon \}$ is therefore calculated on $50$ initializations. Each time, the reconstruction giving the lowest RMSE$(\calY)$ among the five tries is kept. We set $R = 2$ during all the experiments. \\

\noindent
Quantitative results are collected in Table \ref{tab:csc_res_noise_app}. The most remarkable result is that even under a very strong noise (SNR of $5$dB), \algo yields good reconstructions and is not too dependant on the initialization as shown by the two last metrics. Figure \ref{fig:csc_noisy_app} provides a visual example of this important property. We see that \algo reconstructs the input signal with high accuracy while FCSC-ShM completely fails and mostly overfits the noise. This was expected for \algo because the noise does not share the low-rank structure of the signal. Hence, the K-CSC model which includes a low-rank constraint ``manages'' to not capture it and thus recovers the true signal with accuracy. In other words, taking into account the low-rank structure of the signal eliminates the noise and allows a better recovery of the activations. Furthermore, notice that since for both datasets $R^* = 2$ , the larger the signal, the more restrictive the rank constraint is. This leads to lower RMSEs on the large-scale dataset than on the small-scale one.

\begin{table}\scriptsize
	\caption{Results return on the CSC task on the two datasets with noise. For \algo, $R$ is set to the true value, $R^* = 2$. Mean and standard deviation are reported.}
	\label{tab:csc_res_noise_app}
	\centering
	\ra{1.3}
	\begin{adjustbox}{max width=\textwidth}
		\begin{tabular}{@{}ll|Hl|Hl@{}} \toprule
			\textbf{CSC \qquad $\mathbf{R=2}$}&& \multicolumn{2}{l}{\textbf{\emph{Noisy small-scale dataset}}} & \multicolumn{2}{l}{\textbf{\emph{Noisy large-scale dataset}}} \\
			\textbf{SNR} & \textbf{Metrics} & \textbf{T-ConvADMM} & \textbf{\algo} & \textbf{T-ConvADMM} & \textbf{\algo} \\
			\midrule
			$25$dB & RMSE$(\calY)$ $\downarrow$ & $\mathbf{3.988 \cdot e{-4} \,(\pm 4.121 \cdot e{-5})}$ & $\mathbf{3.999 \cdot e{-4} \,(\pm 4.427 \cdot e{-5})}$ & $\mathbf{6.523 \cdot e{-5} \,(\pm 5.970 \cdot e{-7})}$ & $\mathbf{7.606 \cdot e{-5} \,(\pm 1.253 \cdot e{-6})}$ \\
			& RMSE$(\calZ)$ $\downarrow$ & $\mathbf{2.397 \cdot e{-4} \,(\pm 2.331 \cdot e{-5})}$ & $\mathbf{2.403 \cdot e{-4} \,(\pm 2.528 \cdot e{-5})}$ & $\mathbf{3.820 \cdot e{-5} \,(\pm 4.431 \cdot e{-7})}$ & $\mathbf{4.469 \cdot e{-5} \,(\pm 1.1666 \cdot e{-6})}$ \\
			& $\#\{$RMSE$(\calY) < 1. e{-3}\}$ $\uparrow$ & $\mathbf{98 \%}$ & $\mathbf{94 \%}$ &  $\mathbf{86 \%}$ &  $\mathbf{90 \%}$ \\
			& $\#\{$RMSE$(\calZ) < 1. e{-3}\}$ $\uparrow$ & $\mathbf{98 \%}$ & $\mathbf{96 \%}$ & $\mathbf{86 \%}$ & $\mathbf{90 \%}$ \\
			\midrule
			
			$10$dB & RMSE$(\calY)$ $\downarrow$ & $\mathbf{2.513 \cdot e{-3} \,(\pm 1.046 \cdot e{-4})}$ & $\mathbf{2.492 \cdot e{-3} \,(\pm 9.024 \cdot e{-5})}$ & $\mathbf{4.254 \cdot e{-4} \,(\pm 9.016 \cdot e{-6})}$ & $\mathbf{4.958 \cdot e{-4} \,(\pm 7.733 \cdot e{-6})}$ \\
			& RMSE$(\calZ)$ $\downarrow$ & $1.509 \cdot e{-3} \,(\pm 6.113 \cdot e{-5})$ & $\mathbf{1.495 \cdot e{-3} \,(\pm 5.066 \cdot e{-5})}$ & $\mathbf{2.504 \cdot e{-4} \,(\pm 8.241 \cdot e{-6})}$ & $\mathbf{2.913 \cdot e{-4} \,(\pm 7.194 \cdot e{-4})}$ \\
			& $\#\{$RMSE$(\calY) < 2.5 e{-3}\}$ $\uparrow$ & $\mathbf{84 \%}$ & $\mathbf{84 \%}$ &  $84 \%$ &  $\mathbf{88 \%}$ \\
			& $\#\{$RMSE$(\calZ) < 2.5 e{-3}\}$ $\uparrow$ & $\mathbf{96 \%}$ & $\mathbf{98 \%}$ & $84 \%$ & $\mathbf{90 \%}$ \\
			\midrule
			
			$5$dB & RMSE$(\calY)$ $\downarrow$ & $5.224 \cdot e{-3} \,(\pm 3.302 \cdot e{-4})$ & $\mathbf{4.847 \cdot e{-3} \,(\pm 3.166 \cdot e{-4})}$ & $\mathbf{8.835 \cdot e{-4} \,(\pm 2.140 \cdot e{-5})}$ & $\mathbf{9.918 \cdot e{-4} \,(\pm 1.529 \cdot e{-5})}$ \\
			& RMSE$(\calZ)$ $\downarrow$ & $3.147 \cdot e{-3} \,(\pm 2.039 \cdot e{-4})$ & $\mathbf{2.894 \cdot e{-3} \,(\pm 1.805 \cdot e{-4})}$ & $\mathbf{5.187 \cdot e{-4} \,(\pm 1.708 \cdot e{-5})}$ & $\mathbf{5.828 \cdot e{-4} \,(\pm 1.434 \cdot e{-5})}$ \\
			& $\#\{$RMSE$(\calY) < 5. e{-3}\}$ $\uparrow$ & $41 \%$ & $\mathbf{52 \%}$ &  $84 \%$ &  $\mathbf{88 \%}$ \\
			& $\#\{$RMSE$(\calZ) < 4. e{-3}\}$ $\uparrow$ & $84 \%$ & $\mathbf{86 \%}$ & $84 \%$ & $\mathbf{88 \%}$ \\
			\toprule
			\textbf{CSC}&& \multicolumn{2}{l}{\textbf{\emph{Noisy small-scale dataset}}} & \multicolumn{2}{l}{\textbf{\emph{Noisy large-scale dataset}}} \\
			\textbf{SNR} & \textbf{Metrics} & \textbf{FCSC-SM} \citep{bristow2013fast, wohlberg2015efficient} & \textbf{ConvFISTA-FD} \citep{chalasani2013fast, wohlberg2015efficient} & \textbf{FCSC-SM} \citep{bristow2013fast, wohlberg2015efficient} & \textbf{ConvFISTA-FD} \citep{chalasani2013fast, wohlberg2015efficient} \\
			\midrule
			$25$dB & RMSE$(\calY)$ $\downarrow$ & $2.292 \cdot e{-3} \,(\pm 1.220 \cdot e{-5})$ & $2.109 \cdot e{-3} \,(\pm 2.547 \cdot e{-4})$ & $1.732 \cdot e{-3} \,(\pm 8.707 \cdot e{-6})$ & $1.732 \cdot e{-3} \,(\pm 8.703 \cdot e{-6})$ \\
			& RMSE$(\calZ)$ $\downarrow$ & $1.454 \cdot e{-3} \,(\pm 7.326 \cdot e{-5})$ & $1.311 \cdot e{-3} \,(\pm 2.027 \cdot e{-4})$ & $1.050 \cdot e{-3} \,(\pm 2.794 \cdot e{-6})$ & $1.050 \cdot e{-3} \,(\pm 2.791 \cdot e{-5})$ \\
			& $\#\{$RMSE$(\calY) < 1. e{-3}\}$ $\uparrow$ & $0\%$ & $0\%$ & $0\%$ & $0\%$ \\
			& $\#\{$RMSE$(\calZ) < 1. e{-3}\}$ $\uparrow$ &$0\%$ & $0\%$ & $10\%$ & $10\%$ \\

			\midrule
			$10$dB & RMSE$(\calY)$ $\downarrow$ & $6.734 \cdot e{-3} \,(\pm 7.117 \cdot e{-4})$ & $6.673 \cdot e{-3} \,(\pm 6.847 \cdot e{-4})$ & $6.689 \cdot e{-3} \,(\pm 3.344 \cdot e{-5})$ & $6.689 \cdot e{-3} \,(\pm 3.344 \cdot e{-4})$ \\
			& RMSE$(\calZ)$ $\downarrow$ & $4.393 \cdot e{-3} \,(\pm 6.060 \cdot e{-4})$ & $4.367 \cdot e{-3} \,(\pm 5.919 \cdot e{-4})$ & $4.405 \cdot e{-3} \,(\pm 3.011 \cdot e{-3})$ & $4.406 \cdot e{-3} \,(\pm 3.010 \cdot e{-3})$ \\
			& $\#\{$RMSE$(\calY) < 2.5 e{-3}\}$ $\uparrow$ & $0\%$ & $0\%$ & $0\%$ & $0\%$ \\
			& $\#\{$RMSE$(\calZ) < 2.5 e{-3}\}$ $\uparrow$ &$0\%$ & $0\%$ & $0\%$ & $0\%$ \\

			\midrule
			$5$dB & RMSE$(\calY)$ $\downarrow$ & $1.215 \cdot e{-2} \,(\pm 1.252 \cdot e{-3})$ & $1.209 \cdot e{-2} \,(\pm 1.202 \cdot e{-3})$ & $1.211 \cdot e{-2} \,(\pm 6.243 \cdot e{-4})$ & $1.211 \cdot e{-2} \,(\pm 6.242 \cdot e{-4})$ \\
			& RMSE$(\calZ)$ $\downarrow$ & $7.800 \cdot e{-3} \,(\pm 1.033 \cdot e{-3})$ & $7.774 \cdot e{-3} \,(\pm 1.009 \cdot e{-3})$ & $7.812 \cdot e{-3} \,(\pm 5.252 \cdot e{-4})$ & $7.813 \cdot e{-3} \,(\pm 5.251 \cdot e{-4})$ \\
			& $\#\{$RMSE$(\calY) < 5. e{-3}\}$ $\uparrow$ & $0\%$ & $0\%$ & $0\%$ & $0\%$ \\
			& $\#\{$RMSE$(\calZ) < 4. e{-3}\}$ $\uparrow$ &$0\%$ & $0\%$ & $0\%$ & $0\%$ \\
			\toprule
			\textbf{CSC}&& \multicolumn{2}{l}{\textbf{\emph{Noisy small-scale dataset}}} & \multicolumn{2}{l}{\textbf{\emph{Noisy large-scale dataset}}} \\
			\textbf{SNR} & \textbf{Metrics} & \textbf{FCSC-ShM} \citep{bristow2013fast, wohlberg2015efficient} & \textbf{FCSC-ShM} \citep{bristow2013fast, wohlberg2015efficient} & \textbf{FCSC-ShM} \citep{bristow2013fast, wohlberg2015efficient} & \textbf{FCSC-ShM} \citep{bristow2013fast, wohlberg2015efficient} \\
			\midrule
			$25$dB & RMSE$(\calY)$ $\downarrow$ & $2.292 \cdot e{-3} \,(\pm 1.220 \cdot e{-5})$ & $2.292 \cdot e{-3} \,(\pm 1.220 \cdot e{-5})$ & $1.732 \cdot e{-3} \,(\pm 8.707 \cdot e{-6})$ & $1.732 \cdot e{-3} \,(\pm 8.707 \cdot e{-6})$ \\
			& RMSE$(\calZ)$ $\downarrow$ & $1.454 \cdot e{-3} \,(\pm 7.326 \cdot e{-5})$ & $1.454 \cdot e{-3} \,(\pm 7.326 \cdot e{-5})$ & $1.050 \cdot e{-3} \,(\pm 2.794 \cdot e{-6})$ & $1.050 \cdot e{-3} \,(\pm 2.794 \cdot e{-6})$ \\
			& $\#\{$RMSE$(\calY) < 1. e{-3}\}$ $\uparrow$ & $0\%$ & $0\%$ & $0\%$ & $0\%$ \\
			& $\#\{$RMSE$(\calZ) < 1. e{-3}\}$ $\uparrow$ &$0\%$ & $0\%$ & $10\%$ & $10\%$ \\

			\midrule
			$10$dB & RMSE$(\calY)$ $\downarrow$ & $6.734 \cdot e{-3} \,(\pm 7.117 \cdot e{-4})$ & $6.734 \cdot e{-3} \,(\pm 7.117 \cdot e{-4})$ & $6.689 \cdot e{-3} \,(\pm 3.344 \cdot e{-5})$ & $6.689 \cdot e{-3} \,(\pm 3.344 \cdot e{-5})$ \\
			& RMSE$(\calZ)$ $\downarrow$ & $4.393 \cdot e{-3} \,(\pm 6.060 \cdot e{-4})$ & $4.393 \cdot e{-3} \,(\pm 6.060 \cdot e{-4})$ & $4.405 \cdot e{-3} \,(\pm 3.011 \cdot e{-3})$ & $4.405 \cdot e{-3} \,(\pm 3.011 \cdot e{-3})$ \\
			& $\#\{$RMSE$(\calY) < 2.5 e{-3}\}$ $\uparrow$ & $0\%$ & $0\%$ & $0\%$ & $0\%$ \\
			& $\#\{$RMSE$(\calZ) < 2.5 e{-3}\}$ $\uparrow$ &$0\%$ & $0\%$ & $0\%$ & $0\%$ \\

			\midrule
			$5$dB & RMSE$(\calY)$ $\downarrow$ & $1.215 \cdot e{-2} \,(\pm 1.252 \cdot e{-3})$ & $1.215 \cdot e{-2} \,(\pm 1.252 \cdot e{-3})$ & $1.211 \cdot e{-2} \,(\pm 6.243 \cdot e{-4})$ & $1.211 \cdot e{-2} \,(\pm 6.243 \cdot e{-4})$ \\
			& RMSE$(\calZ)$ $\downarrow$ & $7.800 \cdot e{-3} \,(\pm 1.033 \cdot e{-3})$ & $7.800 \cdot e{-3} \,(\pm 1.033 \cdot e{-3})$ & $7.812 \cdot e{-3} \,(\pm 5.252 \cdot e{-4})$ & $7.812 \cdot e{-3} \,(\pm 5.252 \cdot e{-4})$ \\
			& $\#\{$RMSE$(\calY) < 5. e{-3}\}$ $\uparrow$ & $0\%$ & $0\%$ & $0\%$ & $0\%$ \\
			& $\#\{$RMSE$(\calZ) < 4. e{-3}\}$ $\uparrow$ &$0\%$ & $0\%$ & $0\%$ & $0\%$ \\
			\bottomrule
		\end{tabular}
	\end{adjustbox}
\end{table}

\section{Additional results on the EEG dataset}
We provide additional results and illustrations for the EEG dataset.

\subsection{Identification of bad channels}
As stated in the main paper, the K-CDL model provides a highly interpretable decomposition of activations. In our setting, the first mode corresponds to the ``channel activations''. Looking at this mode, we see from the activation of one of the learnt atom that three common channels are defective (see Figure \ref{fig:eeg_bad_channel_app}). Fortunately, due to the rank constraint, we are able to reconstruct these channels. To emphasize this phenomenon, we choose a non-deficient channel which is spatially close to a deficient one. We assume that this good channel approximately reflects the values that the bad one would have taken. Therefore, we expect that the reconstruction of the bad channel is similar to the values of the good one. An illustration of the reconstruction is given in Figure \ref{fig:eeg_bad_channel_recovery_app} meaning that the two band of frequencies become less important with time.

\begin{figure}
	\centering
	\includegraphics[width=.2\linewidth]{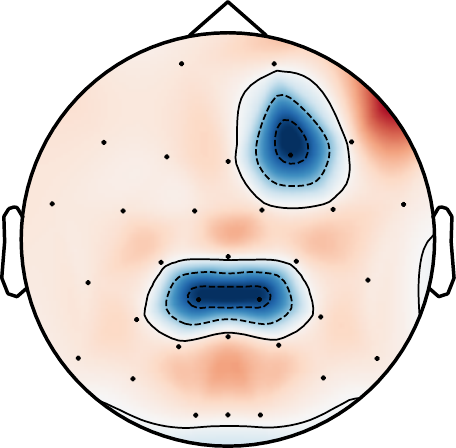}
	\hspace{2em}
	\includegraphics[width=.7\linewidth]{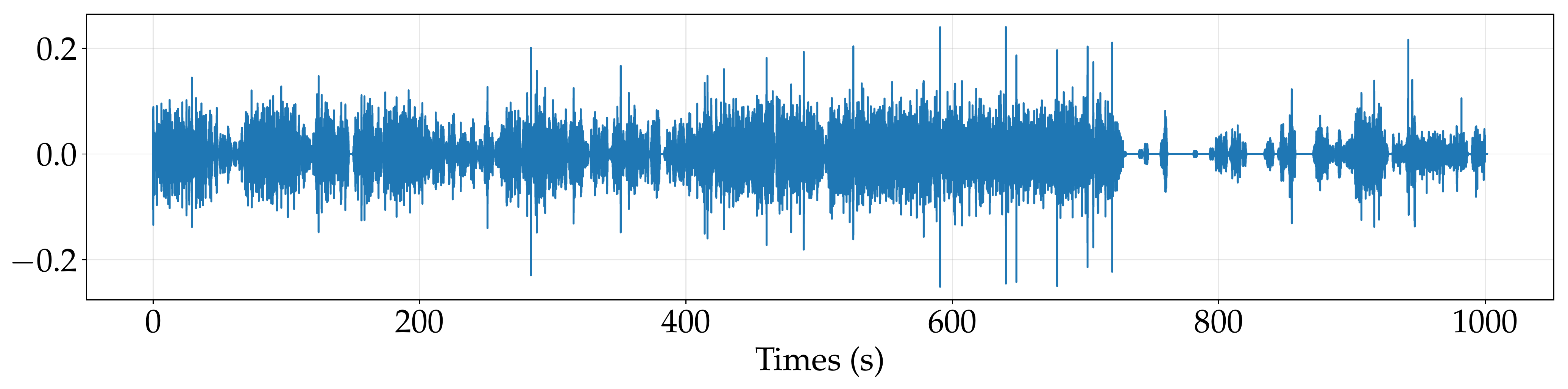}
	\caption{(Left) Spatial activations of the on of the learnt atom before removing the bad channels. (Right) Raw signal at one of the bad channel.}
	\label{fig:eeg_bad_channel_app}
\end{figure}

\begin{figure}
	\centering
	\includegraphics[width=.8\linewidth]{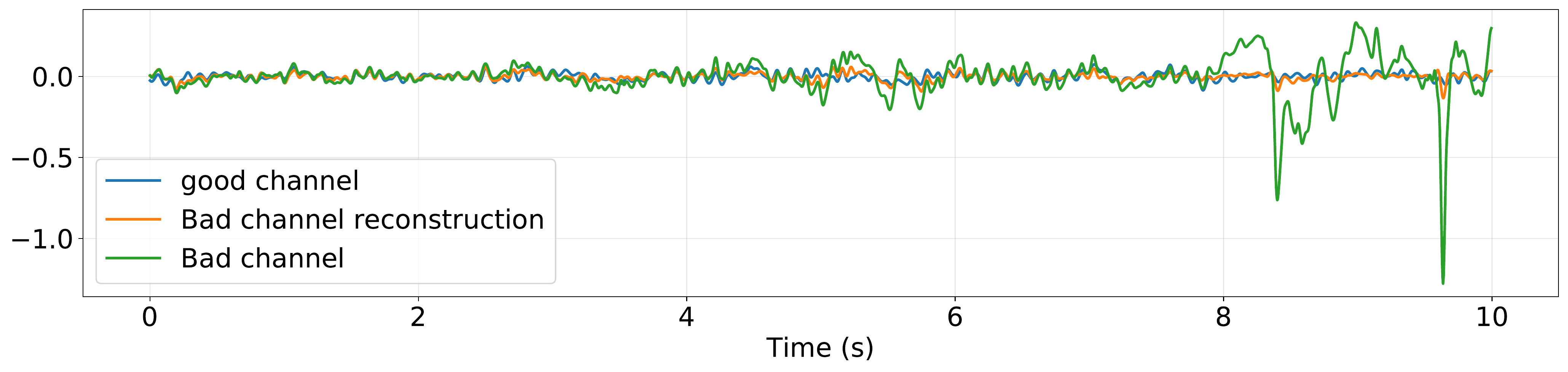}
	\includegraphics[width=.8\linewidth]{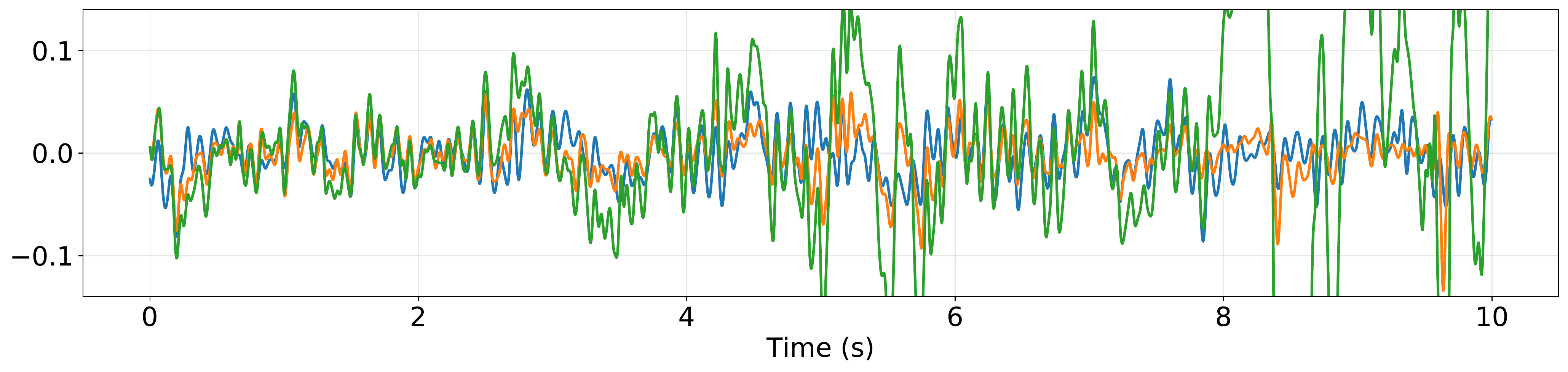}
   \includegraphics[width=.8\linewidth]{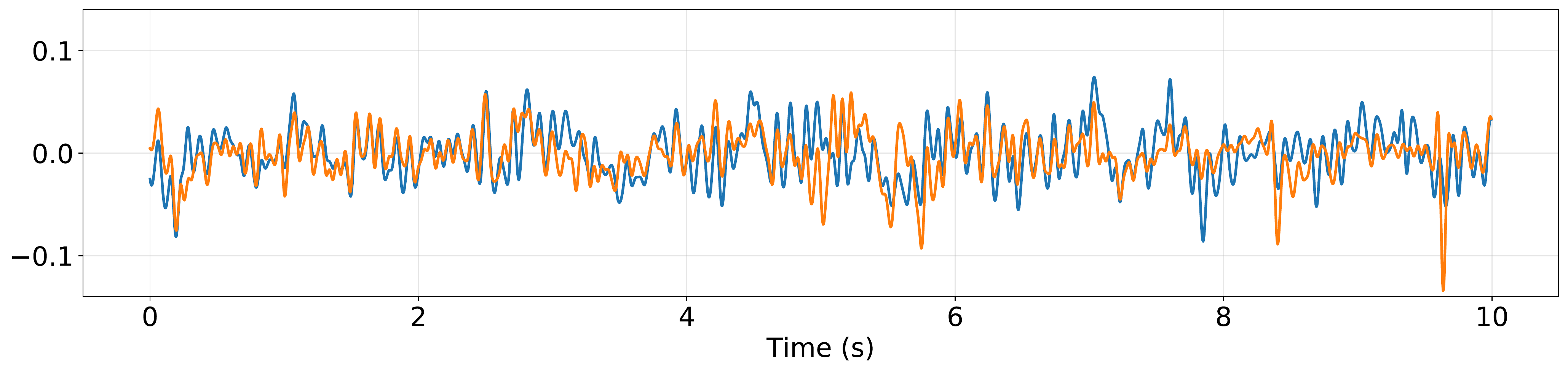}
	\caption{On the top, raw signal of a good channel (blue), a bad channel (green), and a reconstruction of the bad channel (orange). On the middle, a closer look to the good channel. On the bottom, only the good channel (blue) and the reconstruction of the bad channel (orange).}
	\label{fig:eeg_bad_channel_recovery_app}
\end{figure}

\subsection{$\boldsymbol{\delta}$ and $\boldsymbol{\alpha}$-waves during a general anesthesia}
From the second mode (relative to the frequency domain), we have identified atoms link to $\delta$ and $\alpha$-waves, the two most important band of frequencies in General Anesthesia (GA) \citep{purdon2013electroencephalogram}. We know from \citep{purdon2013electroencephalogram} that the importance of these bands during the Recovery Of Consciousness (ROC) tends to decrease in time. Interestingly, this phenomenon is captured by \algo. Indeed, Figure \ref{fig:eeg_alpha_beta_app} shows that the time activations related to these two bands decreases with time.

\begin{figure}
	\centering
    \includegraphics[width=.8\linewidth]{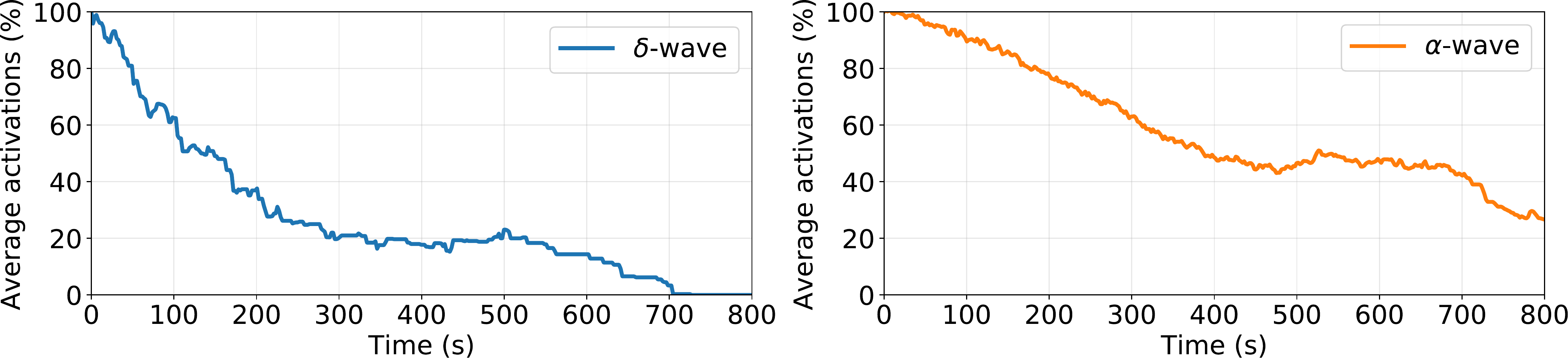}
	\caption{Evolution of the time activations for the first and second atoms of Figure \ref{fig:eeg_atom_small_app} which are relative to the $\delta$ and $\alpha$ waves.}
	\label{fig:eeg_alpha_beta_app}
\end{figure}

\subsection{Removing atoms from the final reconstruction}
In the main paper, we identified the third atom as an atom relative to some impulsion noise. Thanks to its identification, we can remove it from the final reconstruction. We therefore obtain a signal which do not contain this impulsion noise. An illustration is given in Figure \ref{fig:eeg_filtering_app}. Note that, this particular noise is low-rank. Hence, the K-CDL is not robust to it, contrary to the Gaussian noise from the synthetic section. Atoms and activations are displayed in Figure \ref{fig:eeg_atom_small_app}.

\begin{figure}
	\centering
	\includegraphics[width=.8\linewidth]{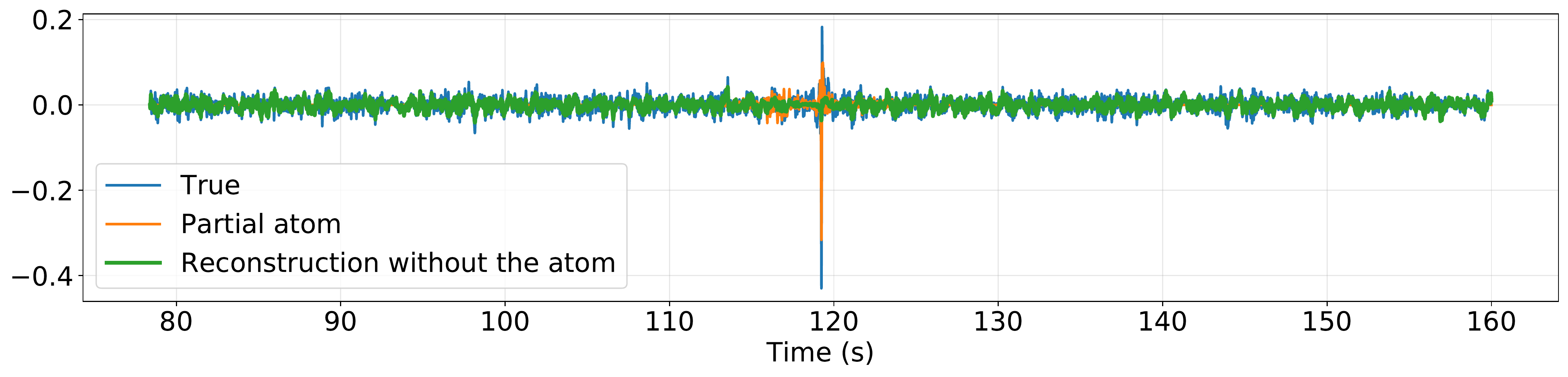}
	\caption{Reconstruction of the signal when one atom relative to the impulsion noise (at $120$s) is removed.}
	\label{fig:eeg_filtering_app}
\end{figure}

\begin{figure}[H]
	\centering
    \includegraphics[width=.82\linewidth]{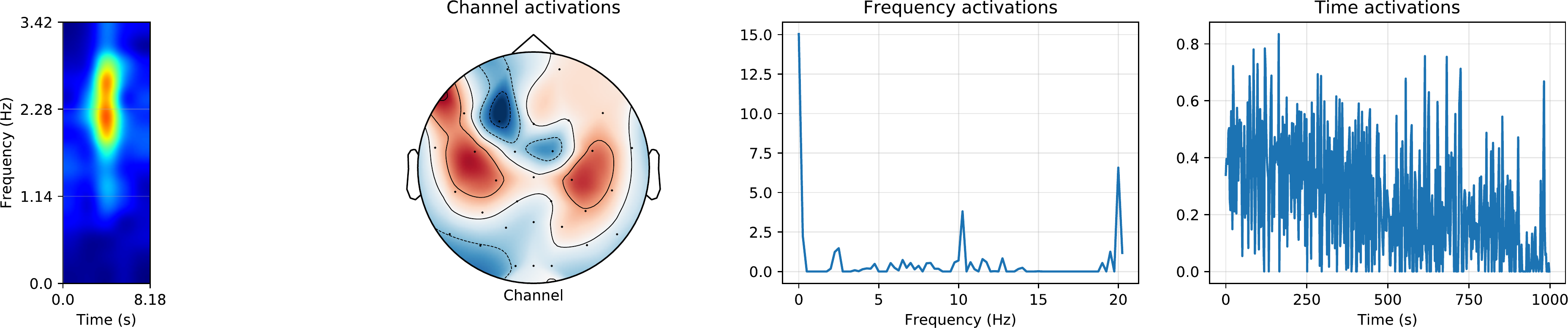}
	\includegraphics[width=.82\linewidth]{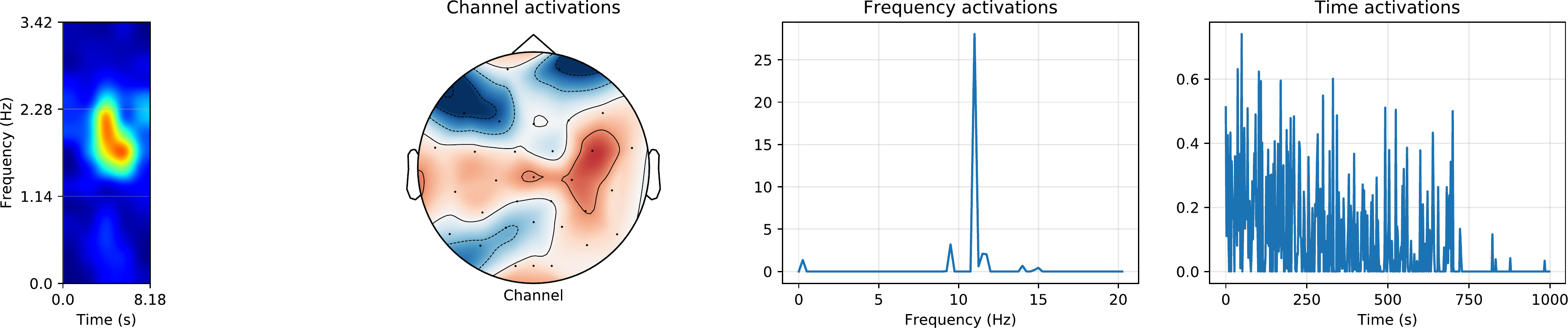}
	\includegraphics[width=.82\linewidth]{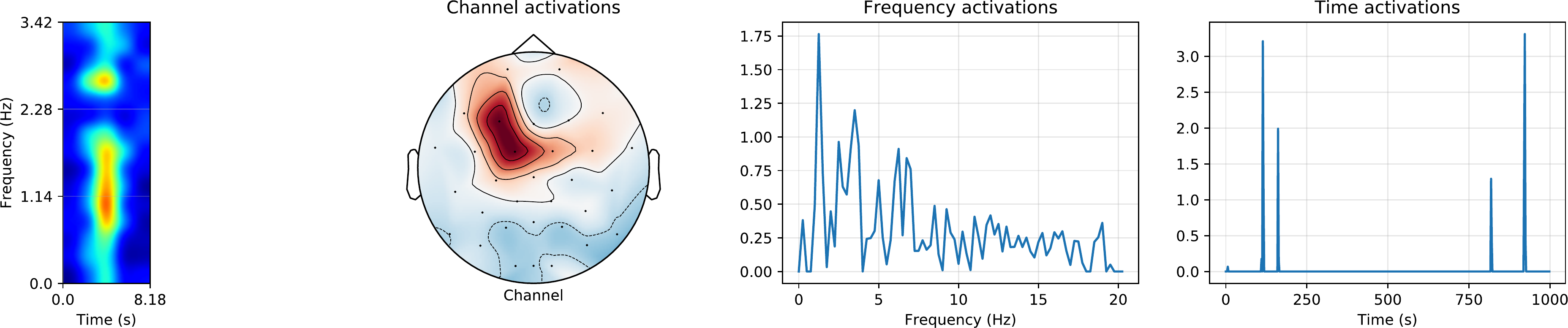}
	\includegraphics[width=.82\linewidth]{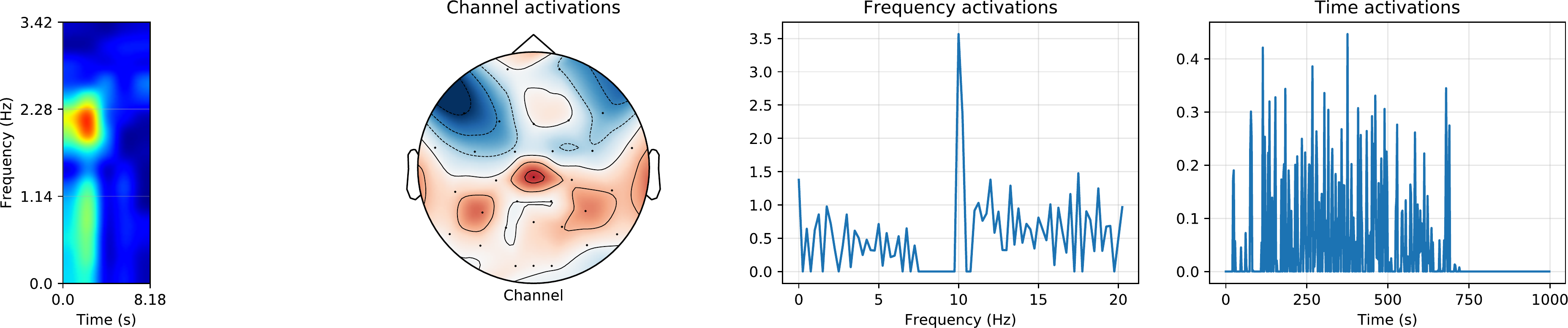}
	\caption{The four atoms with their partial activations. From left to right: the time-frequency atom, the channel activations (mode $1$) on the scalp, the frequency activations (mode $2$), and the time activations (mode $3$). Scale not the same for each figure.}
	\label{fig:eeg_atom_small_app}
\end{figure}

\section{Technical proofs}
This section provides the technical proofs of the different propositions exposed in the main paper. To be as complete as possible, we recall well-known but important definitions and theoretical properties from multidimensional Fourier analysis. We begin with two important definitions.\\

\begin{definition}\emph{(Discrete Fourier Transform (DFT)) --}
	Let consider a function $\calF$ defined on $\{0 \ldots, N_1-1\} \times \cdots \times \{0 \ldots, N_p-1\}$ with period $(N_1, \cdots, N_p)$. The Discrete Fourier Transform (DFT) of $\calF$ is given by
	\begin{equation*}
		\widehat{\calF}[k_1, \cdots, k_p] = \sum_{n_1=0}^{N_1-1} \cdots \sum_{n_p=0}^{N_p-1} \calF[n_1, \ldots, n_p] \exp \left({- i 2 \pi \left(\dfrac{k_1 n_1}{N_1}, \ldots, \dfrac{k_p n_p}{N_p}\right)}\right) \; ,
	\end{equation*}
	and the Inverse DFT (IDFT) of $\widehat{\calF}$ is given by
	\begin{equation*}
		\calF[n_1, \cdots, n_p] = \left(\dfrac{1}{\prod^p_{i=1} N_i}\right) \cdot \sum_{k_1=0}^{N_1-1} \cdots \sum_{k_p=0}^{N_p-1} \widehat{\calF}[k_1, \ldots, k_p] \exp \left({i 2 \pi \left(\dfrac{k_1 n_1}{N_1}, \ldots \dfrac{k_p n_p}{N_p}\right)}\right) \; .
	\end{equation*}
\end{definition}

Let now consider the periodization of two discrete function $\calF$ and $\calG$,
\begin{align*}
\tilde{\calF}[n_1, \cdots, n_p] &= \calF[n_1 \; \text{mod} \;  N_1, \cdots, n_p \; \text{mod} \;  N_p] \\
\tilde{\calG}[n_1, \cdots, n_p] &= \calG[n_1 \; \text{mod} \;  N_1, \cdots, n_p \; \text{mod} \;  N_p] \; .
\end{align*}
The two functions $\tilde{\calF}$ and $\tilde{\calG}$ are now two discrete functions with period $(N_1, \cdots, N_p)$ (each of the modes are periodic one-dimensional signals). The \textit{circular convolution} is defined as follow. \\

\begin{definition}\emph{(Circular discrete convolution) --}
	Let consider two functions $\calF, \calG$ defined on $\{0 \ldots, N_1-1\} \times \cdots \times \{0 \ldots, N_p-1\}$ with both a period of $(N_1, \cdots, N_p)$. The circular convolution between $\tilde{\calF}$ and $\tilde{\calG}$ is given by
	\begin{equation*}
	(\tilde{\calF} \ostar \tilde{\calG})[n_1, \cdots, n_p] = \sum_{k_1=0}^{N_1-1} \cdots \sum_{k_p=0}^{N_p-1} \tilde{\calF}[k_1, \cdots, k_p] \tilde{\calG}[n_1 - k_1, \cdots, n_p - k_p] \; .
	\end{equation*}
\end{definition}
\noindent
$\tilde{\calF} \ostar \tilde{\calG}$ is a signal of period $(N_1, \cdots, N_p)$ and can be decomposed in a Fourier basis like classical periodic signals which give rises to the following important theorem. \\

\begin{theorem} \emph{(Discrete convolution theorem) --}
\label{thm:dct}
	If $\calF$ and $\calG$ have period $(N_1, \cdots, N_p)$, then the DFT of $\calH = \calF \ostar \calG$ is
	\begin{equation*}
	\widehat{\calH}[n_1, \cdots, n_p] = \widehat{\calF}[n_1, \cdots, n_p] \ast \widehat{\calG}[n_1, \cdots, n_p] \; , \quad \text{or in tensor notation} \quad	\widehat{\calH} = \widehat{\calF} \ast \widehat{\calG} \; ,
	\end{equation*}
	where $\ast$ is the component-wise product or Hadamard product.
\end{theorem}

We now present a simple lemma showing the important advantage of separable signals (i.e. low-rank signal) over non-separable ones in term of complexity. \\

\begin{lemma}\emph{(Mode-wise DFT) --}
\label{lem:mwd}
	Let $\calX$ be a tensor in $\IX$ with CP-decomposition $[\![\bX^{(1)}, \ldots, \bX^{(p)}]\!]$. The DFT of $\calX$ can be performed mode-wise i.e.
	\begin{align}
	\widehat{\calX} = \sum_{r=1}^{R} \widehat{\bx}_r^{(1)} \circ \cdots \circ \widehat{\bx}_r^{(p)} = [\![\widehat{\bX}^{(1)}, \ldots, \widehat{\bX}^{(p)}]\!]\; ,
	\end{align}
	where the DFT is taken along each columns of the factors. The complexity of the computation of $(\widehat{\bX}^{(1)}, \cdots, \widehat{\bX}^{(p)})$ using the Fast Fourier Transform algorithm (FFT) is of $\calO(R \sum_{i=1}^{p}n_i \log(n_i)))$ instead of $\calO(\prod_{i=1}^{p}n_i \log(\prod_{i=1}^{p}n_i))$.
\end{lemma}
\begin{proof}
	Using the definition of both the DFT and the CP-decomposition, the proof is straightforward. Furthermore, as we only perform $1$-D FFT, we obtain the given complexity.
\end{proof}

To prove the new results, we will extensively used matricization and vectorization techniques. One important formula in tensor algebra is given by the above proposition.
\begin{proposition}\emph{(Matricization of the Kruskal operator \citep{kolda2009tensor}) --}
	\label{prop:Mat_kruskal}
	Let $\calX$ be a tensor in $\IX$ with CP-decomposition $[\![\bX^{(1)}, \ldots, \bX^{(p)}]\!]$. Then, 
	\begin{align*}
		\bX_{(q)} &= \bX_q \left(\bX_p \odot \cdots \odot \bX_{q+1} \odot \bX_{q-1} \odot \cdots \odot \bX_1 \right)^\transpose
		= \bX_q \left(\stackrel{\hookleftarrow}{\odot}^p_{i=1, i\neq q} \bX_i \right)^\transpose \; ,
	\end{align*}
	where $\odot$ is the Khatri–Rao product and $\stackrel{\hookleftarrow}{\odot}^p_{i=1}$ denotes the product of $p$ Khatri–Rao products in reverse order. We can also also vectorized this formula which gives
	\begin{align*}
	\text{vec}(\bX_{(q)}) &= \left(\bX_p \odot \cdots \odot \bX_{q+1} \odot \bX_{q-1} \odot \cdots \odot \bX_1 \otimes \boldI_{n_q} \right) \text{vec}(\bX_q) \\
	&= \left(\stackrel{\hookleftarrow}{\odot}^p_{i=1, i\neq q} \bX_i  \otimes \boldI_{n_q} \right) \text{vec}(\bX_q)  \; ,
	\end{align*}
	where $\boldI_{n_q}$ is an identity matrix of size $(n_q \times n_q)$. \\
\end{proposition}

\begin{lemma}\emph{(Equality in the Fourier domain) --}
	The fidelity term $f(\cdot)$ from Equation (3) of the main paper is equals in the Fourier domain to
	\begin{equation}
	f\left(\{\bZ_{k, q}\}^K_{k=1}\right) = \dfrac{1}{2 \prod_{i=1}^{p}N_i} \left\lVert \widehat{\calY} - \sum_{k=1}^{K} \widehat{\calD}_k \ast [\![\widehat{\bZ}^{(1)}_{k}, \cdots, \widehat{\bZ}^{(p)}_{k}]\!] \right\rVert^2_F \; .
	\end{equation}
\end{lemma}
\begin{proof}
The proof rests on several equalities and properties.
	\begin{align*}
	&\dfrac{1}{2}\left\lVert \calY - \sum_{k=1}^{K} \calD_k \ostar \sum_{r=1}^{R} \bz^{(1)}_{k, r} \circ \cdots \circ \bz^{(p)}_{k, r} \right\rVert^2_F \\
	&= \dfrac{1}{2\prod_{i=1}^{p}N_i} \left\lVert \widehat{\calY} - \sum_{k=1}^{K} \textit{DFT}(\calD_k \ostar \sum_{r=1}^{R} \bz^{(1)}_{k, r} \circ \cdots \circ \bz^{(p)}_{k, r}) \right\rVert^2_F \quad \text{(Parseval's theorem -- Plancherel)} \\
	&= \dfrac{1}{2\prod_{i=1}^{p}N_i}\left\lVert \widehat{\calY} - \sum_{k=1}^{K} \widehat{\calD}_k \ast \sum_{r=1}^{R} \textit{DFT}(\bz^{(1)}_{k, r} \circ \cdots \circ \bz^{(p)}_{k, r}) \right\rVert^2_F \quad \text{(Convolution theorem \eqref{thm:dct})}\\
	&= \dfrac{1}{2\prod_{i=1}^{p}N_i}\left\lVert \widehat{\calY} - \sum_{k=1}^{K} \widehat{\calD}_k \ast \sum_{r=1}^{R} \widehat{\bz}^{(1)}_{k, r} \circ \cdots \circ \widehat{\bz}^{(p)}_{k, r} \right\rVert^2_F \quad \text{(Lemma \eqref{lem:mwd})} \\
	&= \dfrac{1}{2\prod_{i=1}^{p}N_i}\left\lVert \widehat{\calY} - \sum_{k=1}^{K} \widehat{\calD}_k \ast [\![\widehat{\bZ}^{(1)}_{k}, \cdots, \widehat{\bZ}^{(p)}_{k}]\!] \right\rVert^2_F \quad \text{(Kruskal operator)} \; .
	\end{align*}
\end{proof}

\begin{figure}
	\centering
	\includegraphics[width=0.2\linewidth]{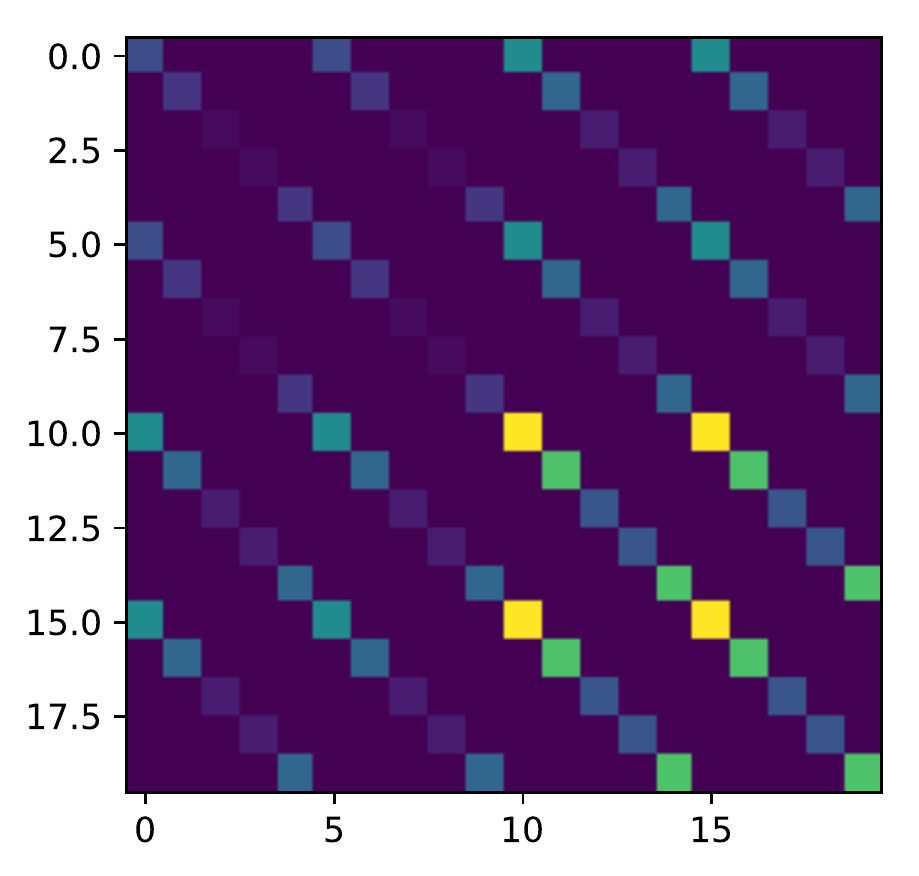}
	\includegraphics[width=0.2\linewidth]{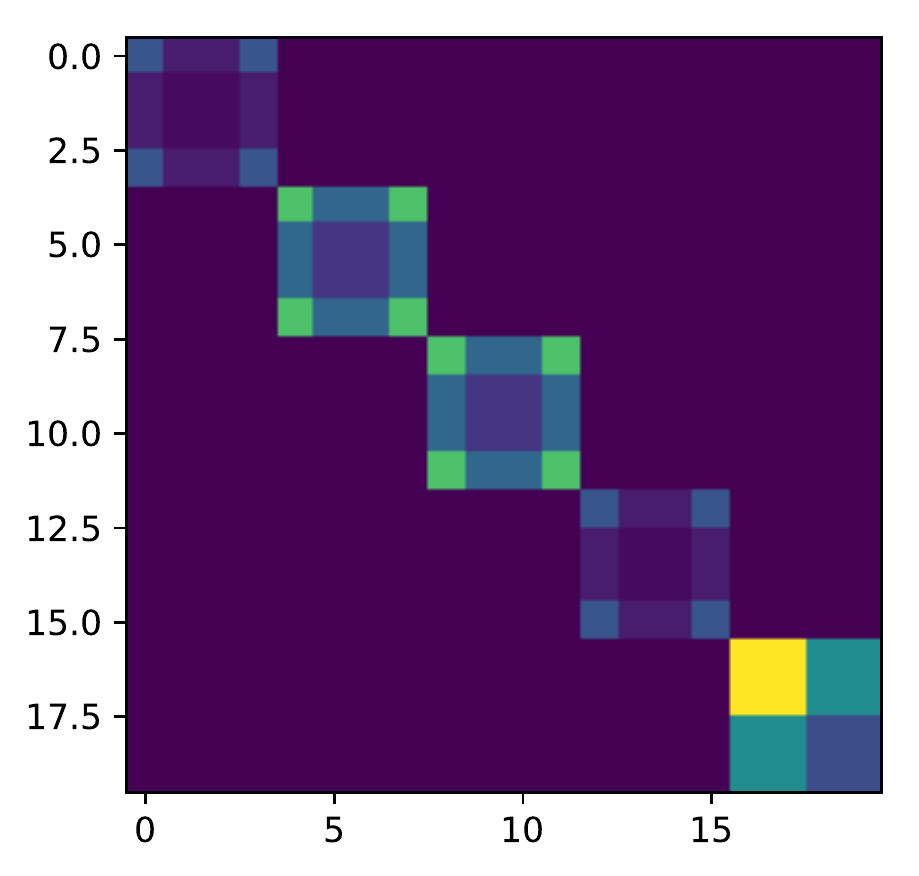}
\hspace{3em}
	\includegraphics[width=0.2\linewidth]{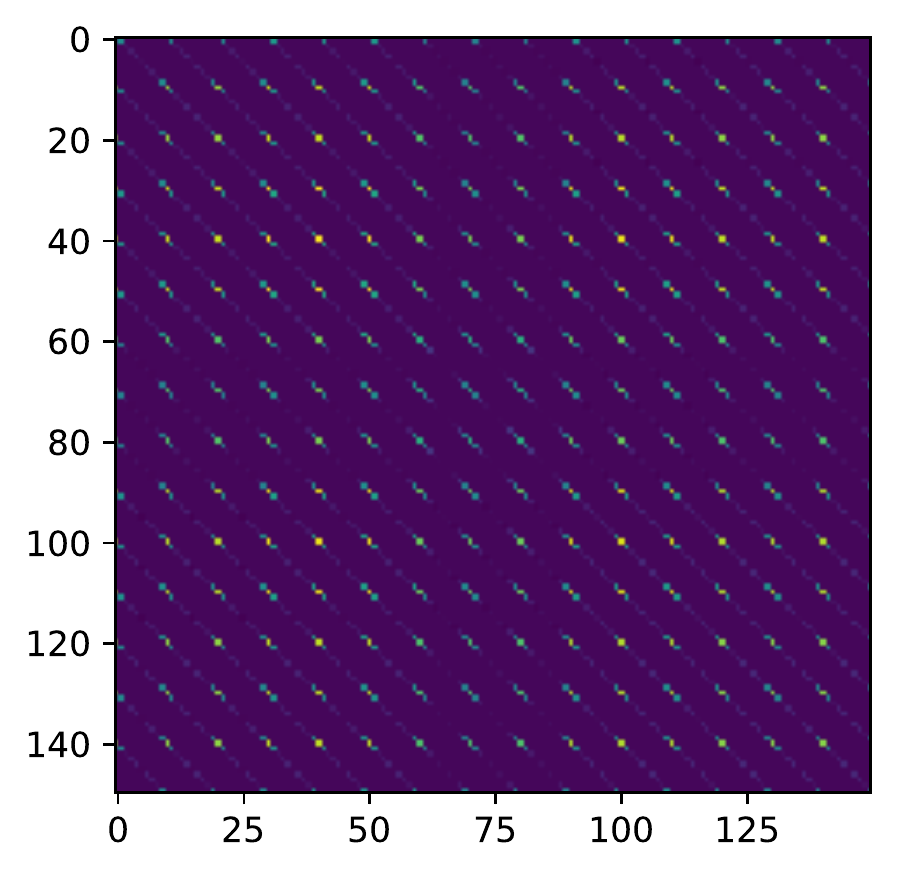}
    \includegraphics[width=0.2\linewidth]{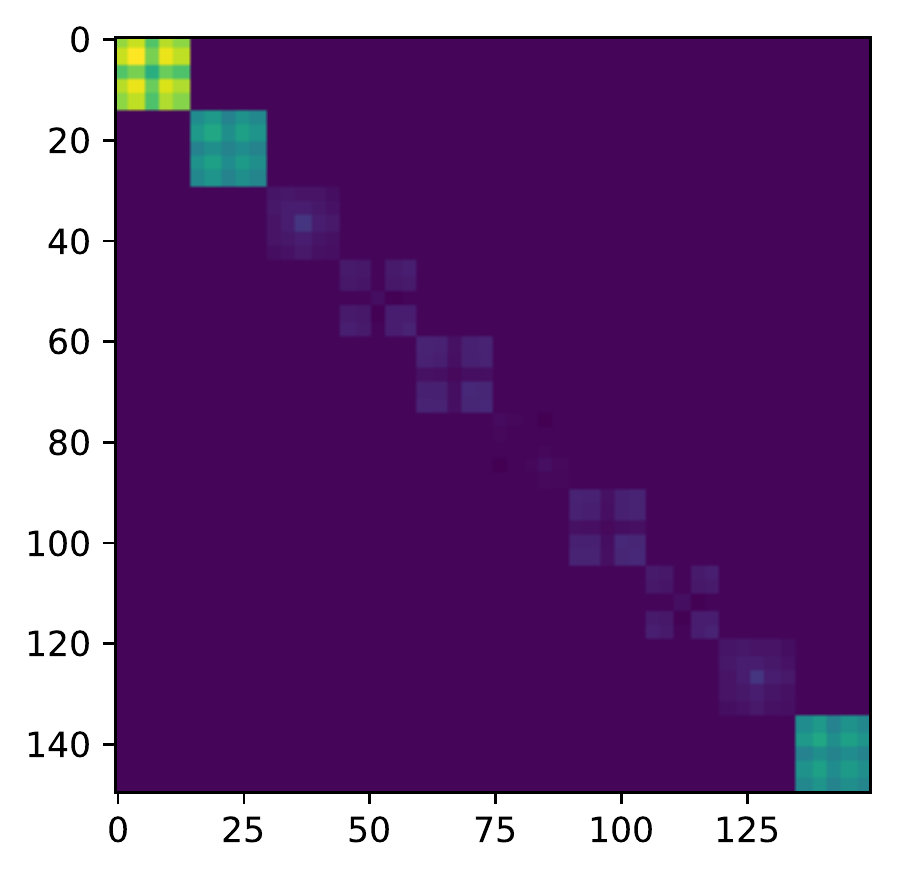}
	\caption{Visualization of $G = (\widehat{\bA}^{H} \otimes \boldI)\widehat{\bGamma}^{H} \widehat{\bGamma}(\widehat{\bA} \otimes \boldI)$ before and after a reordering. The two left matrices correspond to the Gram matrix without and with reordering. The two right matrices also correspond to the Gram matrix without and with reordering but for a higher dimension.}
	\label{fig:bandmat}
\end{figure}

We now state the main theorem of the paper.
\paragraph{Theorem 1.}(Compact vectorized formulation) --
%
\emph{
	\label{prop:vect_form}
    	Let $\widehat{\by}^{(q)}$ (resp. $\widehat{\bd}_k^{(q)} $)  be the vectorization of the folding of $\widehat{\calY}$ (resp. $\widehat{\calD_k}$) along the dimension $q$,
    	$\widehat{\bZ}_{k}^{(q)} = [\widehat{\bZ}_{k}^{(q)}(:, 1) \mid \ldots \mid \widehat{\bZ}_{k}^{(q)}(:, R)]$ be the DFT of $\bZ_{k}^{(q)}$ along its columns and
    	$\widehat{\bz}^{(q)}_{k} = \texttt{vec}(\widehat{\bZ}_{k}^{(q)}).$ 
    	Then 
	\begin{equation}\label{eq:f vector}
	f\left(\{\bZ_{k}^{(r)}\}^{K, p}_{k=1, r=1}\right) \propto \dfrac{1}{2} \Big\lVert \widehat{\by}^{(q)} - \widehat{\bGamma}^{(q)} (\widehat{\bA}^{(q)} \otimes \boldI^{(q)})\widehat{\bz}^{(q)} \Big\rVert^2_F \; ,
	\end{equation}
	where 
	$\widehat{\bz}^{(q)} = [\widehat{\bz}^{(q)^\transpose}_1, \ldots, \widehat{\bz}^{(q)^\transpose}_K]^\transpose \in \IC^{K R n_q}$,
	$\widehat{\bGamma}^{(q)} = [\text{diag}(\widehat{\bd_1}^{(q)}), \ldots, \text{diag}(\widehat{\bd_K}^{(q)})] \in \IC^{n_1 \cdots n_p \times K n_1 \cdots n_p}$,
	$\boldI^{(q)} \in \IR^{n_q \times n_q}$ is the identity matrix
	 and 
	\begin{equation*}
	\widehat{\bA}^{(q)} = \begin{pmatrix} 
	\widehat{\bB}^{(q)}_{1} &  & 0 \\ 
	& \ddots &  \\ 
	0 &  & \widehat{\bB}^{(q)}_{K}
	\end{pmatrix}  \in \IC^{K \prod_{1, i \neq q}^{p} n_i \times K R} \quad \text{with} \quad \widehat{\bB}^{(q)}_{k} = (\stackrel{\hookleftarrow}{\odot}^p_{i=1, i\neq q} \widehat{\bZ}_{k}^{(i)} ) \; .
	\end{equation*}
}

\begin{proof}
	The proof mainly rests on the Proposition \ref{prop:Mat_kruskal} and on the formulation of the previous lemma.
	\begin{align*}
	& \left\lVert \widehat{\calY} - \sum_{k=1}^{K} \widehat{\calD}_k \ast [\![\widehat{\bZ}_{k, 1}, \cdots, \widehat{\bZ}_{k, p}]\!] \right\rVert^2_F = \left\lVert \widehat{\bY}^{(q)} - \sum_{k=1}^{K} \widehat{\bD}^{(q)}_k \ast \widehat{\bZ}_{k, q} \left(\stackrel{\hookleftarrow}{\odot}^p_{i=1} \widehat{\bZ}^{(i)}_{k} \right)^\transpose \right\rVert^2_F \quad \text{(matricization)} \\
	&= \left\lVert \widehat{\by}^{(q)} - \sum_{k=1}^{K} \widehat{\bd}_k^{(q)} \ast \left(\stackrel{\hookleftarrow}{\odot}^p_{i=1} \widehat{\bZ}^{(i)}_{k} \otimes I\right) \text{vec}(\widehat{\bZ}_{k, q}) \right\rVert^2_F \quad \text{(vectorization)} \\
	&= \left\lVert \widehat{\by}^{(q)} - \sum_{k=1}^{K} \text{diag}(\widehat{\bd}_k^{(q)}) \left(\stackrel{\hookleftarrow}{\odot}^p_{i=1} \widehat{\bZ}^{(i)}_{k} \otimes I\right) \text{vec}(\widehat{\bZ}_{k, q}) \right\rVert^2_F \qquad (\bx \ast \by = \text{diag}(\bx)\by)\\
	&= \left\lVert \widehat{\by}^{(q)} - \sum_{k=1}^{K} \text{diag}(\widehat{\bd}_k^{(q)})\widehat{\bC}_{k} \widehat{\bz}_{k} \right\rVert^2_F  \; ,
	\end{align*}
	where the last line is just notations. To obtain the final equality, we stack the matrices $\{\text{diag}(\widehat{\bd}_k^{(q)})\}$ and construct a block-diagonal matrix such that the block are the $\{\bC_k\}$. Finally we obtain the following equality.
	\begin{align*}
	\begin{pmatrix}
	\widehat{\bC}_{1} &  &  \\ 
	& \ddots &  \\ 
	&  & \widehat{\bC}_{K}
	\end{pmatrix} = 
	\begin{pmatrix}
	\widehat{\bB}_{1} \otimes \boldI &  &  \\ 
	& \ddots &  \\ 
	&  & \widehat{\bB}_{K} \otimes \boldI
	\end{pmatrix} =
	\begin{pmatrix}
	\widehat{\bB}_{1} &  &  \\ 
	& \ddots &  \\ 
	&  & \widehat{\bB}_{K}
	\end{pmatrix}  \otimes \boldI \; ,
	\end{align*}
	with $\widehat{\bB}_{k} = (\stackrel{\hookleftarrow}{\odot}^p_{i=1, i\neq q} \widehat{\bZ}_{k, i} )$. This end the proof.
\end{proof}

\paragraph{Corollary 1.}(Gradient of $f$) --
\emph{
With the notation of Theorem 2, the gradient of $f$ with respect to $\bz^{(q)} = [ \texttt{vec}({\bZ}^{(q)}_1)^\transpose, \ldots, \texttt{vec}({\bZ}^{(q)}_K)^\transpose]^\transpose$ is given by
\begin{equation}\label{eq:gradient formula_app}
\nabla_{\bz^{(q)}} 	f\left(\{\bZ_{k}^{(r)}\}^{K, p}_{k=1, r=1}\right) =  \text{IDFT}\left[\left((\widehat{\bA}^{(q)} \otimes \boldI)\widehat{\bGamma}^{(q)}\right)^H\left(\widehat{\bGamma}^{(q)} (\widehat{\bA}^{(q)} \otimes \boldI)\widehat{\bz}^{(q) }- \widehat{\by}^{(q)}\right)\right] \; ,
\end{equation}
where IDFT$[\cdot]$ stands for the Inverse Discrete Fourier Transform.
}
\begin{proof}
    See \citep{liu2018first} \textit{Section 2.3.2.} for a complete and detailed proof.
\end{proof}

\paragraph{Proposition 1.}
\emph{
%
The matrix $\bG \triangleq (\widehat{\bA}^H \otimes \boldI)\widehat{\bGamma}^H\widehat{\bGamma} (\widehat{\bA} \otimes \boldI)$ can be obtained by computing $K^2$ blocks $\bG_{k,\ell}, 1 \le k,\ell \le K,$  where
\begin{equation}
 \bG_{k,\ell} = \left((\stackrel{\hookleftarrow}{\odot}^p_{i=1, i\neq q} \widehat{\bZ}_{k, i} )^H \otimes I\right) \overline{\text{diag}(\widehat{\bd_k}^{(q)})} \text{diag}(\widehat{\bd_\ell}^{(q)}) \left((\stackrel{\hookleftarrow}{\odot}^p_{i=1, i\neq q} \widehat{\bZ}_{\ell, i} ) \otimes I\right) \; ,
\end{equation}
and each of these blocks can be computed in $\calO(R^2 \prod_{i=1, i\neq q}^{p}n_i)$ operations. 
}

\begin{proof}
	The first step of the proof requires to write $\widehat{\bGamma}$ as the Kronecker product of two specific matrices in order to use the equality $(\bA \otimes \bB)(\bC \otimes \bD) = (\bA\bC \otimes \bB\bD)$. Recall that $\widehat{\bGamma}$ is a block-diagonal matrix, i.e. $\widehat{\bGamma} = [\text{diag}(\widehat{\bd_1}^{(q)}), \cdots, \text{diag}(\widehat{\bd_K}^{(q)})]$. We can decompose each diagonal-block, $\text{diag}(\widehat{\bd_k}^{(q)})$, into smaller diagonal matrices using the kronecker product as follow
	\begin{align*}
		\text{diag}(\widehat{\bd_k}^{(q)}) = \sum_{i=1}^{N_{\backslash q}} \text{diag}(e_i) \otimes \Delta_{k, i} \quad \text{with} \quad N_{\backslash q} = \prod_{i=1, i\neq q}^{p} n_i \; ,
	\end{align*}
	where $\text{diag}(e_i) \in \IR^{N_{\backslash q} \times N_{\backslash q}}$ and $\Delta_{k, i} \in \IC^{n_q \times n_q}$ being the $i$-th diagonal block of $\text{diag}(\widehat{\bd_k}^{(q)})$ (i.e. $\Delta_{k, i} = \text{diag}(\widehat{\bd_k}^{(q)})_{(i \cdot n_q:(i+1) \cdot n_q), (i \cdot n_q:(i+1) \cdot n_q)}$). As $(\text{diag}(e_i) \otimes \Delta_{k, i})$ is decomposed into two matrices of the proper dimension, we can used the equality $(\bA \otimes \bB)(\bC \otimes \bD) = (\bA\bC \otimes \bB\bD)$ leading to
	\begin{align*}
&\left((\odot^{p}_{i=1, i\neq q} \widehat{\bZ}_{k, i} )^H \otimes I\right) \overline{\text{diag}(\widehat{\bd_k}^{(q)})} \text{diag}(\widehat{\bd_\ell}^{(q)}) \left((\odot^{p}_{i=1, i\neq q} \widehat{\bZ}_{\ell, i} ) \otimes I\right) \\
&=\left(\widehat{\bB}_{k}^H \otimes I\right) \sum_{i=1}^{N_{\backslash q}} \left(\text{diag}(e_i) \otimes \overline{\Delta_{k, i}}\right) \sum_{j=1}^{N_{\backslash q}} \left(\text{diag}(e_j) \otimes \Delta_{\ell, j}\right) \left(\widehat{\bB}_{\ell} \otimes I\right) \\
&=\sum_{i=1}^{N_{\backslash q}} \sum_{j=1}^{N_{\backslash q}}  \left(\widehat{\bB}_{k}^H \otimes I\right) \left(\text{diag}(e_i) \otimes \overline{\Delta_{k, i}}\right) \left(\text{diag}(e_j) \otimes \Delta_{\ell, j}\right) \left(\widehat{\bB}_{\ell} \otimes I\right) \\
&=\sum_{i=1}^{N_{\backslash q}} \sum_{j=1}^{N_{\backslash q}}\left(\widehat{\bB}_{k}^H\text{diag}(e_i) \text{diag}(e_j) \widehat{\bB}_{\ell} \otimes \overline{\Delta_{k, i}} \Delta_{\ell, j} \right)\\ 
&=\sum_{i=1}^{N_{\backslash q}} \left(\widehat{\bB}_{k}^H\text{diag}(e_i) \text{diag}(e_i) \widehat{\bB}_{\ell} \otimes \overline{\Delta_{k, i}} \Delta_{\ell, i} \right)\\
&=\sum_{i=1}^{N_{\backslash q}} \left((\text{diag}(e_i)\widehat{\bB}_{k})^H \text{diag}(e_i) \widehat{\bB}_{\ell} \otimes \overline{\Delta_{k, i}} \Delta_{\ell, i} \right) =\sum_{i=1}^{N_{\backslash q}} \left(\overline{\widehat{\bB}}_{k}(i, :) \circ \widehat{\bB}_{\ell}(i, :) \otimes \overline{\Delta_{k, i}} \Delta_{\ell, i} \right) \; .
\end{align*}

The outer product of two vectors in $\IC^{1 \times R}$ is of complexity $\calO(R^2)$. This product is made for each $1 \leq i \leq N_{\backslash q}$ and for each $K^2$ blocks. Hence, the overall complexity is $\calO((KR)^2 \prod_{i=1, i\neq q}^{p}n_i)$. In addition, as this matrix has a particular block structure, i.e. this is a band-matrix (see Figure \ref{fig:bandmat}), its product with a vector of size $KRn_q$ is of small complexity equals to $\calO((KR)^2 n_q)$.
\end{proof}

\end{document}